\documentclass[10pt,twocolumn,letterpaper]{article}
\usepackage{soul}
\usepackage{tabularx}
\usepackage{booktabs} %
\usepackage{multirow}
\usepackage{makecell}
\usepackage{upgreek} %
\usepackage{adjustbox}

\usepackage[utf8]{inputenc}
\usepackage{pifont}
\usepackage{xcolor}           %
\usepackage{listings}         %
\usepackage{algorithm}        %
\usepackage{algpseudocode}    %
\usepackage{wrapfig}          %
\usepackage{caption}          %
\usepackage{graphicx}         %
\usepackage{float}            %
\usepackage{microtype}        %

\usepackage{graphicx, animate}
\usepackage{xcolor}

\usepackage[dvipsnames,table]{xcolor}

\usepackage{amsmath, amssymb,array} %
\usepackage{amsthm}
\newtheorem{theorem}{Theorem}
\usepackage{booktabs}         %

\newcolumntype{R}{>{$}r<{$}}
\newcolumntype{L}{>{$}l<{$}}
\newcolumntype{M}{R@{${}={}$}L}

\usepackage[pagenumbers]{cvpr}

\newcommand{\ourmodel}{%
\ifshowedits 
    \textcolor{blue}{\architecturename}%
\else
    \architecturename%
\fi
}

\newif\ifshowedits
\newcommand{\ours}{%
\ifshowedits 
    \textcolor{green}{\methodname}%
\else
    \methodname%
\fi
}
\newif\ifshowedits
\newcommand{\oursshort}{%
\ifshowedits 
    \textcolor{green}{\methodnameshort}%
\else
    \methodnameshort%
\fi
}
\newif\ifshowedits

\newif\ifshowedits

\newif\ifshowedits

\newif\ifshowedits

\NewCommandCopy{\myst}{\st}
\renewcommand{\st}[1]{\ifshowedits{\textcolor{red}{\myst{#1}}}\fi}

\newcommand{\code}[1]{%
    \texttt{\textcolor{black}{#1}}%
}

\definecolor{cvprblue}{rgb}{0.21,0.49,0.74}
\usepackage[pagebackref,breaklinks,colorlinks,allcolors=cvprblue]{hyperref}

\def\paperID{11899} %
\def\confName{CVPR}
\def\confYear{2026}

\title{\ours}

\author{Soumik Mukhopadhyay$^{*}$\\
    {\tt\small soumik@umd.edu}
\and
    Prateksha Udhayanan$^{*}$\\
    {\tt\small pudhayan@umd.edu}
\and
    Abhinav Shrivastava\\
    {\tt\small abhinav2@umd.edu}
\and
    University of Maryland, College Park\\
}

\begin{document}

\newcommand{\methodname}{Scale Space Diffusion}
\newcommand{\methodnameshort}{SSD}
\newcommand{\architecturename}{Flexi-UNet}

\maketitle

\definecolor{forestgreen}{RGB}{34,139,34}

\def\thefootnote{*}\footnotetext{Equal contribution.}

\begin{abstract}
Diffusion models degrade images through noise, and reversing this process reveals an information hierarchy across timesteps. Scale-space theory exhibits a similar hierarchy via low-pass filtering. We formalize this connection and show that highly noisy diffusion states contain no more information than small, downsampled images - raising the question of why they must be processed at full resolution. To address this, we fuse scale spaces into the diffusion process by formulating a family of diffusion models with generalized linear degradations and practical implementations. Using downsampling as the degradation yields our proposed \ours. To support \ours, we introduce \ourmodel{}, a UNet variant that performs resolution-preserving and resolution-increasing denoising using only the necessary parts of the network. We evaluate our framework on CelebA and ImageNet and analyze its scaling behavior across resolutions and network depths. 
Our \href{https://prateksha.github.io/projects/scale-space-diffusion/}{project website} is available publicly.

\end{abstract} 

\addtocontents{toc}{\protect\setcounter{tocdepth}{-1}}
\section{Introduction}
\label{sec:intro}

\begin{figure}[t]
    \centering
    \begin{subfigure}{0.53\linewidth}
        \centering
        \includegraphics[width=\linewidth]{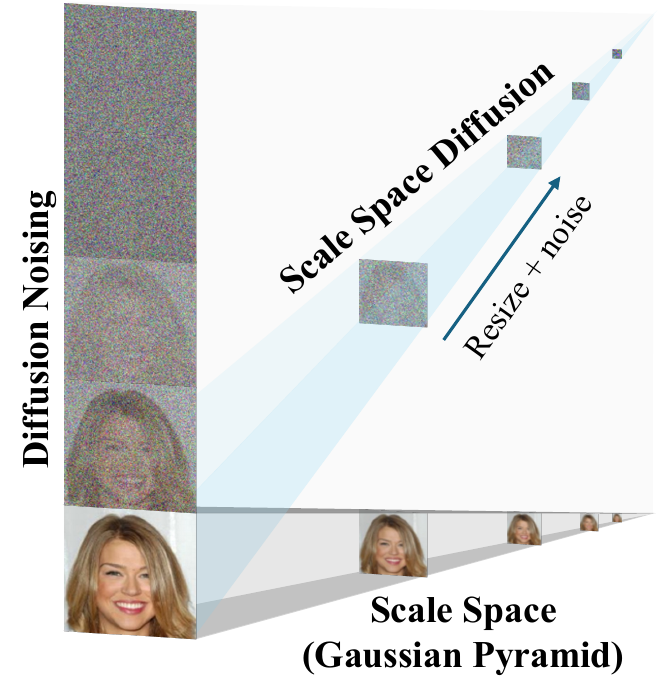}
        \caption{}
        \label{fig:left}
    \end{subfigure}
    \hfill
    \begin{subfigure}{0.46\linewidth}
        \centering
        \includegraphics[width=\linewidth]{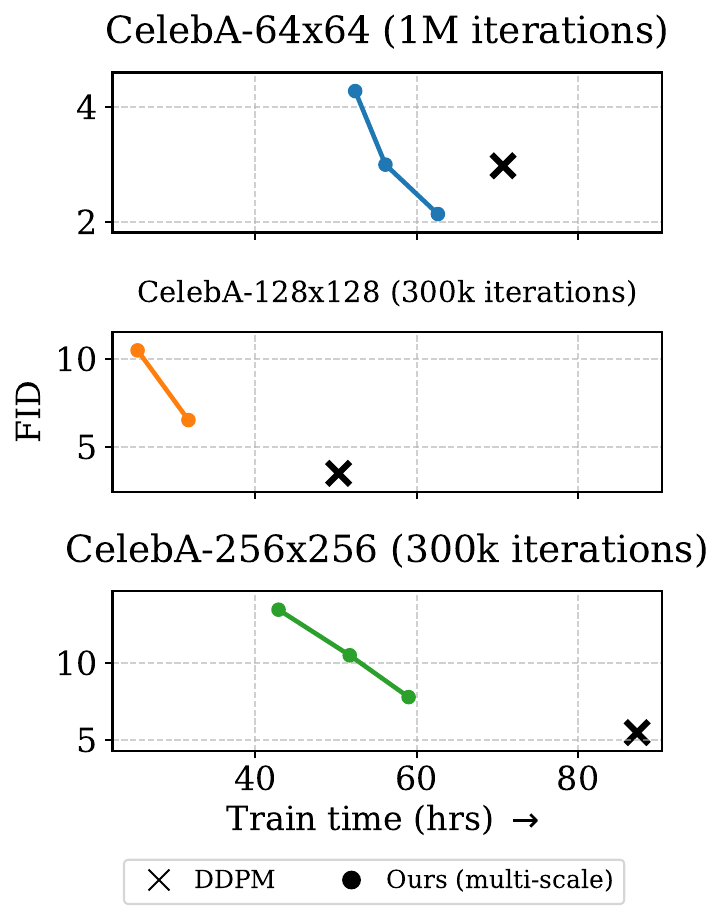}
        \caption{}
        \label{fig:right}
    \end{subfigure}
    \caption{(a) Our proposed \ours~fuses scale spaces into diffusion models. (b) We show trends in image generation performance versus time for our proposed \ourmodel{} for CelebA-64, CelebA-128, and CelebA-256. Multiple point on the same plot represent our models with different number levels (\ie, number of intermediate resolutions). We see immense gains in efficiency with resolution scaling while having reasonable performance. }\vspace{-0.2in}
    \label{fig:teaser}
\end{figure}

Diffusion models~\cite{ddpm, song2020score} are a class of generative models that achieve image synthesis by reversing an iterative noising process. It has been observed that states at different stages of the diffusion process encode different types of information~\cite{mukhopadhyay2024text}. As shown along the y-axis of Fig.~\ref{fig:teaser}(a), increasing diffusion noise progressively removes fine facial details while retaining only coarse structure. Eventually, with sufficient noising, even this structural information is lost.
This illustrates that diffusion timesteps form an intrinsic information hierarchy.

A similar property underlies scale space theory~\cite{lindeberg1994scale}, a fundamental subfield of computer vision. 
Scale spaces also represent image signals in an information-hierarchical manner through successive low-pass filtering. 
Along the x-axis of Fig.~\ref{fig:teaser}(a), we see the loss of details as the resolution decreases in a Gaussian pyramid, mirroring the information dissipation in the diffusion process. The main distinction lies in the mechanism of information degradation: diffusion uses iterative noising, whereas scale spaces use progressive blurring or downsampling.

We investigate this relationship between diffusion and scale spaces formally through a preliminary mathematical modeling of information in both processes. This reveals striking parallels in their information content, suggesting a fundamental connection between the two. \textit{Intuitively, one may ask why completely noisy images should be processed at high resolution when they contain information equivalent to that of a tiny image.} These parallels indicate that the two axes in Fig.~\ref{fig:teaser}(a) correspond to different but compatible ways of information degradation.

In this work, we revisit pixel diffusion to achieve a unification of scale spaces and the diffusion process. 
Previous attempts at this either operate only at the highest resolution~\cite{hoogeboom2022blurring, bansal2022cold}, making them computationally inefficient, or rely on simplistic covariance assumptions~\cite{abu2023udpm} that may not hold in practice, or perform noisy scale shifting using high-frequency~\cite{atzmon2024edify} or decorrelation noise~\cite{jin2024pyramidal, chen2025pixelflow, jeong2025upsample}, which remain inference-time approximations. Unlike pyramidal flow-matching approaches that approximate scale changes only during inference, our formulation integrates scale transitions directly into the diffusion process. 
In contrast, we first develop a mathematical theory for diffusion processes under generalized linear degradations, yielding a \emph{family of diffusion processes}. We further illustrate how these can be implemented in modern deep-learning frameworks. Next, using image resizing as the linear degradation, we realize a fusion of scale spaces with diffusion. We term this process \ours~(\oursshort). Denoising diffusion probabilistic models (DDPM)~\cite{ddpm} emerge as a special case of \oursshort, corresponding to the trivial case of resizing to the same size, \ie, the identity operator. These generalized degradations naturally induce non-isotropic posteriors, which we handle through an implicit sampling procedure.

To realize the general version of \ours, we require a neural network architecture capable of reversing the downsizing degradation, \ie, it must be able to upsample a noisy state. A na\"ive approach could use a UNet~\cite{unet} directly, but this would require even small-scale images to pass through the full network, leading to unnecessary computational cost. To address this, we propose a novel convolutional neural network (ConvNet) architecture that augments the standard UNet to use only the relevant levels of the network. It supports both resolution-preserving diffusion steps and next-resolution upscaling at all stages of a Gaussian pyramid. We denote this architecture as \ourmodel{}.

We analyze our framework and architectures on unconditional image generation using commonly used datasets of CelebA~\cite{liu2015faceattributes} and ImageNet~\cite{deng2009imagenet}. To study the scaling properties of our method, we conduct experiments at multiple resolutions of CelebA dataset as shown in Fig.~\ref{fig:teaser} (b). We observe that our models are faster during both training and inference while achieving reasonable FID scores. The key contributions of this work are:
\begin{enumerate}
    \item We uncover and analyze the relationship between the states of diffusion models and the levels of scale spaces.
    \item We build the mathematical foundation for a family of generalized linear diffusion processes, and techniques to implement them in modern deep-learning frameworks. With resizing as the choice for the linear degradation, we realize the fusion of diffusion and scale spaces, which we term \ours{}.  
    \item To enable \ours{}, we introduce a novel architecture \ourmodel{} capable of handling both resolution-changing as well as resolution-preserving reverse diffusion across multiple resolutions. 
\end{enumerate}

\addtocontents{toc}{\protect\setcounter{tocdepth}{2}} %

\addtocontents{toc}{\protect\setcounter{tocdepth}{-1}}

\section{Related Work}
\label{sec:related_work}

\textbf{Diffusion Models.} Diffusion Models have become the de facto standard for image generation in recent times. Early works such as DDPM~\cite{ddpm} achieved high-quality image generation  without adversarial training, but relied on simulating a Markov chain with a large number of steps for sampling. DDIM~\cite{song2020ddim} accelerated the sampling process, while methods such as LDM~\cite{rombach2022ldm} performed denoising in a compact latent space rather than directly in the pixel-space. Recently, DiTs have become popular, replacing traditional UNet based backbones with transformer architecture~\cite{peebles2023dit}. Motivated by the goal of scaling diffusion models for high-resolution image generation while maintaining architectural simplicity and high-frequency image details, we propose an end-to-end \ours{} model that performs denoising directly in the pixel domain.

\noindent\textbf{Scale-Space Theory.}  Scale-space theory~\cite{lindeberg1994scale} is a fundamental concept in computer vision, that provides a framework for multi-scale image representation and analysis. It has been widely used in visual understanding tasks~\cite{lowe1999object, canny2009computational}. The underlying idea of representation at multiple scales has been smartly used in the context of generative models to progressively generate images at increasing resolutions. In GAN-based approaches, Progressive GAN ~\cite{karras2017progressivegan} has shown excellent results in generating high-resolution images by learning to generate at increasing resolutions during the training process. In some other works such as LAPGAN~\cite{denton2015lapgan}, multiple GANs, one for each scale, are used to upscale the image by producing a
residual, similar to a Laplacian pyramid. 

Several works in the space of diffusion models have also drawn inspiration from scale-space theory. Cascaded diffusion model~\cite{ho2021cascadeddm} consists of a series of diffusion models that generate images of increasing resolutions, where the base model produces a low-resolution image and subsequent super-resolution models refine it using the upsampled version of the low-resolution image as a condition. Matryoshka Diffusion~\cite{gu2023matryoshka} model proposes a diffusion process that denoises inputs at multiple resolutions jointly.

However, none of these approaches directly incorporate scale-space theory in the diffusion process because the noise component of the noisy intermediate state leads to correlated noise pixels at an upsampled state. Some works solve this by adding additional noise at the higher resolution. Relay Diffusion~\cite{teng2023relay} imagines a low-resolution generation as a high-resolution image with block noise and trains a model to denoise it at higher resolution with a weighted combination of block noise and high-resolution noise. Laplacian Diffusion Models~\cite{atzmon2024edify} train separate models for different resolutions and add a Laplacian residual of high-resolution noise during resolution transitions. However, simply adding high-resolution noise does not fully resolve the distribution mismatch between noisy states at different resolutions. Pyramidal Flow Matching~\cite{jin2024pyramidal} addresses this issue by adding decorrelated noise while also rolling the diffusion process back to a noisier timestep. PixelFlow~\cite{chen2025pixelflow} and Region Adaptive Latent Sampling~\cite{jeong2025upsample} build on this idea. Bottleneck-Sampling~\cite{tian2025bottleneck} as opposed to increasing scales introduces a bottleneck scale for better generation, while Decomposed Flow Matching~\cite{haji2025decomposable} predicts Laplacian residuals of clean images. 
UDPM~\cite{abu2023udpm} tries to add blurring and subsampling into the diffusion process, assumes isotropic posterior covariance to simplify their reverse diffusion derivation, which may not hold, given that the blurring kernels usually overlap in most implementations of resizing. We show through \ours{} that end-to-end training of a single diffusion model capable of handling multiple resolutions, with a generalized mathematical formulation for resolution transitions, helps to achieve faster generation, while preserving high-quality.

\addtocontents{toc}{\protect\setcounter{tocdepth}{2}} %

\addtocontents{toc}{\protect\setcounter{tocdepth}{-1}}
\section{Scale Spaces vis-\`a-vis Diffusion Timesteps}
\label{sec:analysis}
In this section, we outline the motivation behind our approach, which originates from a simple but compelling intuition.
Consider the intermediate states of a diffusion model (Fig.~\ref{fig:info}(a), bottom) and the scales of a Gaussian pyramid (Fig.~\ref{fig:info}(b), bottom). 
If one squints and focuses on the third image from the left along the $t$-axis, the overall structure of the face begins to emerge, which is remarkably similar to the information present in the images corresponding to smaller spatial scales along the $r$-axis of the Gaussian pyramid. As we move rightward along either axis (\ie, decreasing $t$ or increasing $r$), it becomes evident how finer details are added progressively. 

This observation suggests a striking correspondence in the information hierarchy between diffusion timesteps and scale-space resolutions (or scales). Our goal is to quantify this correspondence. To do so, we first review the standard diffusion process, and then formalize our intuition by mathematically characterizing the amount of information present across diffusion states.

\subsection{Preliminary: Standard Diffusion Process}
In standard denoising diffusion probabilistic models (DDPM)~\cite{ddpm}, the forward diffusion process is modeled as a Markov chain that progressively noises a signal by adding Gaussian noise. For $x_0 \sim q(x_0)$, where $q(x_0)$ is the data distribution, the process is defined as:
\begin{equation}
\label{eq:forward_step_ddpm}
\begin{split}
x_t = \sqrt{\alpha_t}x_{t-1} +\sqrt{1-\alpha_t}\epsilon, \quad \epsilon \sim \mathcal{N}(0, \textbf{I})
\end{split}
\end{equation}
where  ${\{\beta_t\}}_{t=1}^T$ is the variance schedule (with $\alpha_t:=1-\beta_t$).
This expression, when applied iteratively over $t$, leads to an alternative definition that expresses the noisy state as a linear combination of the signal $x_0$ and the noise $\epsilon$:
\begin{equation}
\label{eq:make_noisy_ssd_analysis}
\begin{split}
x_t = \sqrt{\Bar{\alpha}_t}x_0 +\sqrt{1-\Bar{\alpha}_t}\epsilon, \quad \epsilon \sim \mathcal{N}(0, \textbf{I})
\end{split}
\end{equation}
where $\Bar{\alpha}_t:=\prod_{i=0}^t\alpha_i$. Diffusion models aim to reverse this process by approximating the posterior distribution $q(x_{t-1}|x_0,x_t)$ using a neural network (with parameters $\theta$) that predicts the noise $\epsilon$ in Eq.~\ref{eq:make_noisy_ssd_analysis}. This model, $\epsilon_\theta(x_t, t)$, is trained using a simplified loss function $\mathcal{L}_{\text{simple}}=\mathbb{E}_{x_0,t,\epsilon}\left[\|\epsilon_\theta(x_t, t) - \epsilon \|_2^2\right]$. The model can also be parameterized to predict $x_0$ instead of $\epsilon$.

\subsection{Information Degradation in Diffusion and Scale Spaces}

\begin{figure}[t]
    \centering
    \includegraphics[width=0.8\linewidth]{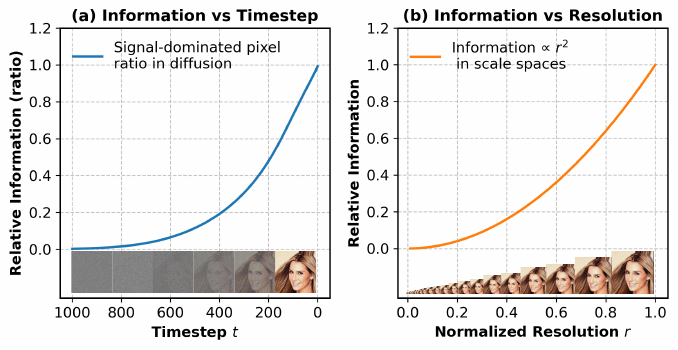}
    \vspace{-0.2in}
    \caption{\textbf{Information Analysis.} (a) Amount of information present in a diffusion state as diffusion step $t$ changes. (b) Amount of information present in images at various resolutions (scales).}
    \vspace{-0.1in}
    \label{fig:info}
\end{figure}

\noindent \textbf{Diffusion States.} In this section, we formally model the information degradation over the diffusion process. Eq.~\ref{eq:make_noisy_ssd_analysis} has two terms -- a signal term and a noise term. One way to model the amount of information present in $x_t$ is to compute the percentage of pixels for which the noise term dominates the signal term,
\ie, $|\sqrt{1-\Bar{\alpha}_t}\,\epsilon| > |\sqrt{\Bar{\alpha}_t}\,x_0|$.
In other words, we are looking for the probability that $|\epsilon|$ is greater than $s(t)\,|x_0|$, where $s(t)=\frac{\sqrt{\Bar{\alpha}_t}}{\sqrt{1-\Bar{\alpha}_t}}$ is the square root of the signal-to-noise coefficient ratio.
We have
$P\left(|\epsilon| > s(t)\,|x_0|\right) = \left(1-\Phi\!\left(s(t)\,|x_0|\right)\right) + \Phi\!\left(-s(t)\,|x_0|\right) = 2\,\Phi\!\left(-s(t)\,|x_0|\right)$,
where $\Phi$ is the CDF of the standard normal distribution, and the second equality follows from symmetry of the standard normal distribution. Hence, $P\left(|\epsilon| \leq s(t)\,|x_0|\right) = 1 - 2\,\Phi\!\left(-s(t)\,|x_0|\right)$. Now to obtain the expected fraction of signal-dominated pixels, \emph{a proxy for information}, we average this probability over the data distribution $q(x_0)$.
For simplicity, let us assume $x_0 \sim \mathcal{U}(-1,1)$. Then the variation of information over timestep $t$ can be written as: 
\begin{equation*}
\begin{split}
  \text{Info}(t)&=\mathbf{E}_{x_0\sim\mathcal{U}(-1,1)}[1 - 2\Phi( - s(t) |x_0|)]\\
  &=1- 2\int_{-1}^1p_{\mathcal{U}(-1,1)}(x)\Phi(-s(t)|x|)dx\\
  &=1- 2\int_{-1}^1\frac{1}{2}\Phi(-s(t)|x|)dx\\
  &=1- \int_{-1}^1\Phi(-s(t)|x|)dx=1 - 2\int_{0}^1\Phi(-s(t)x)dx,
  \end{split}
\end{equation*}
where we use the fact that the uniform distribution has density $p_{\mathcal{U}} = \frac{1}{2}$ over $[-1,1]$, and for the final equality we split the integral about $x=0$.
Using this simplification, $\text{Info}(t)$ can be numerically computed as a function of $t$, as shown in Fig.~\ref{fig:info}(a).

\noindent \textbf{Scale Spaces.} Similar to the approximation of information across diffusion steps, here we want to approximate the information as a function of image resolution (\ie, scale).
A simple way to model this is to assume:
$$\text{Info} \propto \text{Area}.$$
Let us consider a normalized resolution $r \in [0,1]$, where $0$ represents no pixels and $1$ represents the highest resolution. Under this assumption, the information can be written as:
$$\text{Info}(r) = r^2.$$
This implies, for example, that if the spatial dimensions of an image are halved, then the information becomes one quarter, which may not be strictly true due to redundancy in pixel space. However, the monotonic trend should still hold. This trend is visualized in Fig.~\ref{fig:info}(b).

Notice how there is a similarity in the trends of information degradation as $t$ increases versus as $r$ decreases. This analysis quantifies our main intuition regarding the similarity in the information trends across diffusion steps and scale spaces. Given this insight, we aim to leverage this intuition to construct a framework that realizes scale spaces within the current formulation of diffusion models.

In our initial attempts to incorporate scale spaces into diffusion models, we tried to frame this problem as jumping across the same timesteps of independent diffusion processes at varying scales. However, this led to an accumulation of errors during the iterative inference procedure, resulting in suboptimal outputs. Methods such as Pyramidal Flow Matching~\cite{jin2024pyramidal, chen2025pixelflow, jeong2025upsample} address this issue by adding decorrelation noise when transitioning across scales and then backtracking in time so that an appropriate noise level is selected. This strategy helps mitigate the error accumulation. Nonetheless, it does not actually resolve the underlying issue -- the diffusion process itself is not mathematically modeled to handle scale changes. In this work, we aim to fill this gap.

\addtocontents{toc}{\protect\setcounter{tocdepth}{2}} %

\addtocontents{toc}{\protect\setcounter{tocdepth}{-1}}

\definecolor{cone}{RGB}{220,0,0}
\definecolor{ctwo}{RGB}{0,0,200}
\definecolor{cthree}{RGB}{180,0,180}
\definecolor{cfour}{RGB}{34,139,34}
\definecolor{cfive}{RGB}{255,140,0}    %
\definecolor{csix}{RGB}{0,160,160}     %

\newcommand{\swatch}[1]{%
  \begingroup
  \setlength{\fboxsep}{0pt}%
  \colorbox{#1}{\rule{0pt}{0.6em}\hspace{0.6em}}%
  \endgroup
}

\definecolor{cone}{RGB}{220,0,0}
\definecolor{ctwo}{RGB}{0,0,200}
\definecolor{cthree}{RGB}{180,0,180}
\definecolor{cfour}{RGB}{34,139,34}
\definecolor{cfive}{RGB}{255,140,0}
\definecolor{csix}{RGB}{0,160,160}

\newcommand{\mhl}[3][18]{%
  \begingroup
  \setlength{\fboxsep}{1.0pt}%
  \colorbox{#2!#1}{$\displaystyle #3$}%
  \endgroup
}

\begin{table*}[t]
\centering
\small
\renewcommand{\tabcolsep}{4pt}
\renewcommand{\arraystretch}{1.6}

\caption{Comparison between the formulations of the forward, marginal, and posterior distributions of DDPM and Blurring Diffusion (BD) against our \ours. For Blurring Diffusion, we use `$\text{a}$' instead of $\alpha$ used in their paper, to not confuse it with the $\alpha$ in DDPM. 
Note that BD applies a change of variable $u_t=V^Tx_t$, where $V^T$ is the Discrete Cosine Transform, before performing diffusion, \ie, diffusion in frequency space.
BD and DDPM have equivalent formulations when $\text{a}_t=\sqrt{\Bar{\alpha}_t}$ and $\sigma_t=\sqrt{1-\Bar{\alpha}_t}$. 
While the formulations share structural similarities, \ours{} extends the framework to support general linear degradations (\eg, downscaling), which are not handled by DDPM or BD.
We highlight analogous terms with consistent background colors for easier correspondences across different formulations.
}
\vspace{-0.28in}
\rightline{%
\footnotesize
\textbf{Legend:}\;
Forward \swatch{cthree!30} mean, \swatch{cfour!30} var\;
\quad
Marginal \swatch{cone!30} mean, \swatch{ctwo!30} var\;
\quad
Posterior \swatch{cfive!30} mean, \swatch{csix!30} var
}%

\label{tab:linear_operator_diffusion}

\begin{adjustbox}{max width=\linewidth}
\vspace{0.4em}
\begin{tabular}{@{}l|c|c|c@{}}
\toprule
\textbf{Distributions} & \textbf{DDPM~\cite{ddpm}} & \textbf{Blurring Diffusion~\cite{hoogeboom2022blurring}} & \textbf{\ours} \\ 
\midrule
\midrule

Forward  $q(x_t|x_{t-1})$ 
& 
\(
\begin{aligned}
\small
x_t &= \mhl{cthree}{\sqrt{\alpha_t}}\,x_{t-1}
     + \mhl{cfour}{\sqrt{1-\alpha_t}}\,\epsilon,
\ \epsilon \sim \mathcal{N}(0, \mathbf{I}) \\
& \\
& \\
\end{aligned}
\)
& 
\(
\begin{aligned}
\small
u_t &= \mhl{cthree}{\text{a}_{t|t-1}}\,u_{t-1}
     + \mhl{cfour}{\sigma_{t|t-1}}\,\epsilon,
\ \epsilon \sim \mathcal{N}(0, \mathbf{I}) \\
\text{where }\ \mhl{cthree}{\text{a}_{t|t-1}}
&= \frac{\mhl{cone}{\text{a}_t}}{\mhl{cone}{\text{a}_{t-1}}}, \\
\mhl{cfour}{\sigma_{t|t-1}}
&= \mhl{ctwo}{\sigma_t^2} - \mhl{cthree}{\text{a}_{t|t-1}^2}\,\mhl{ctwo}{\sigma_{t-1}^2}
\end{aligned}
\)
& 
\(
\begin{aligned}
\small
x_t &= \mhl{cthree}{\mu_{t|t-1}}
     + \mhl{cfour}{\eta_{t|t-1}},\ 
\mhl{cfour}{\eta_{t|t-1}} \sim \mathcal{N}(0, \mhl{cfour}{\Sigma_{t|t-1}}) \\
\text{where }\ \mhl{cthree}{\mu_{t|t-1}}
&= \mhl{cthree}{M_t} x_{t-1}
= \mhl{cone}{M_{1:t}}(\mhl{cone}{M_{1:t-1}})^{-1}x_{t-1},\\
\mhl{cfour}{\Sigma_{t|t-1}}
&= \mhl{ctwo}{\Sigma_t} - \mhl{cthree}{M_t}\,\mhl{cfour}{\Sigma_{t-1}}\,\mhl{cthree}{M_t^T}
\end{aligned}
\)
\\

\midrule

Marginal  $q(x_t|x_0)$ 
& 
\(
\begin{aligned}
\small
x_t &= \mhl{cone}{\sqrt{\Bar{\alpha}_t}}\,x_0
     + \mhl{ctwo}{\sqrt{1-\Bar{\alpha}_t}}\,\epsilon,
\ \epsilon \sim \mathcal{N}(0, \mathbf{I})\\
& \\
\end{aligned}
\)
& 
\(
\begin{aligned}
\small
u_t &= \mhl{cone}{\text{a}_t}\,u_0
     + \mhl{ctwo}{\sigma_t}\,\epsilon,
\ \epsilon \sim \mathcal{N}(0, \mathbf{I})\\
& \\
\end{aligned}
\)
& 
\(
\begin{aligned}
\small
x_t &= \mhl{cone}{\mu_t}
     + \mhl{ctwo}{\eta_t},\ 
\mhl{ctwo}{\eta_t} \sim \mathcal{N}(0, \mhl{ctwo}{\Sigma_t})\\
\text{where }\ \mhl{cone}{\mu_t}
&= \mhl{cone}{M_{1:t}}x_0,\qquad
\mhl{ctwo}{\Sigma_t} = \mhl{ctwo}{\sigma_t^2}\mathbf{I}
\end{aligned}
\)
\\

\midrule

Posterior  $q(x_{t-1}|x_t, x_0)$ 
&
\(
\begin{aligned}
\small
x_{t-1} &= \mhl{cfive}{\Tilde{\mu}_{t-1}}
        + \mhl{csix}{\Tilde{\beta}_{t-1}}\,\epsilon,\ 
\epsilon \sim \mathcal{N}(0, \mathbf{I}) \\
\text{where }\ \mhl{csix}{\Tilde{\beta}_{t-1}}
&= \tfrac{1-\Bar{\alpha}_{t-1}}{1-\Bar{\alpha}_t}\beta_t \\
&= \Big(\tfrac{1}{\mhl{ctwo}{1-\Bar{\alpha}_{t-1}}} + \tfrac{\mhl{cthree}{\alpha_t}}{\mhl{cfour}{1-\alpha_t}}\Big)^{-1}, \\
\mhl{cfive}{\Tilde{\mu}_{t-1}}
&=\tfrac{\sqrt{\alpha_t}(1-\Bar{\alpha}_{t-1})}{1-\Bar{\alpha}_t}x_t
+ \tfrac{\Bar{\alpha}_{t-1}\beta_t}{1-\Bar{\alpha}_t}x_0 \\
&= \mhl{csix}{\Tilde{\beta}_{t-1}}
\Big(\tfrac{\mhl{cthree}{\sqrt{\alpha_t}}}{\mhl{cfour}{1-\alpha_t}}x_t
+\tfrac{\mhl{cone}{\Bar{\alpha}_{t-1}}}{\mhl{ctwo}{1-\Bar{\alpha}_{t-1}}}x_0\Big)
\end{aligned}
\)
& 
\(
\begin{aligned}
\small
u_{t-1} &= \mhl{cfive}{\mu_{t\rightarrow t-1}}
        + \mhl{csix}{\sigma_{t \rightarrow t-1}}\,\epsilon,\ 
\epsilon \sim \mathcal{N}(0, \mathbf{I}) \\
\text{where }\ \mhl{csix}{\sigma_{t \rightarrow t-1}^2}
&= \left(\tfrac{1}{\mhl{ctwo}{\sigma_{t-1}^2}} + \tfrac{\mhl{cthree}{\text{a}_{t|t-1}^2}}{\mhl{cfour}{\sigma_{t|t-1}^2}}\right)^{-1}, \\
\mhl{cfive}{\mu_{t \rightarrow t-1}}
&= \mhl{csix}{\sigma_{t \rightarrow t-1}^2}
\left(\tfrac{\mhl{cthree}{\text{a}_{t|t-1}}}{\mhl{cfour}{\sigma_{t|t-1}^2}}x_t
+\tfrac{\mhl{cone}{\text{a}_{t-1}}}{\mhl{ctwo}{\sigma_{t-1}^2}}x_0\right)\\
&
\end{aligned}
\)
& 
\(
\begin{aligned}
\small
x_{t-1} &= \mhl{cfive}{\mu_{t \rightarrow t-1}}
        + \mhl{csix}{\eta_{t \rightarrow t-1}},\ 
\mhl{csix}{\eta_{t \rightarrow t-1}} \sim \mathcal{N}(0, \mhl{csix}{\Sigma_{t \rightarrow t-1}}) \\
\text{where }\ \mhl{csix}{\Sigma_{t\rightarrow t-1}}
&= (\mhl{ctwo}{\Sigma_{t-1}^{-1}} + \mhl{cthree}{M_t^T}\,\mhl{cfour}{\Sigma_{t|t-1}^{-1}}\,\mhl{cthree}{M_t})^{-1} \\
&= \sigma_{t-1}^2\mathbf{I}-\tfrac{\sigma_{t-1}^4}{\sigma_t^2}M_t^TM_t,\\
\mhl{cfive}{\mu_{t\rightarrow t-1}}
&= \mhl{csix}{\Sigma_{t\rightarrow t-1}}
(\mhl{cthree}{M_t^T}\,\mhl{cfour}{\Sigma_{t|t-1}^{-1}}x_t
+\mhl{ctwo}{\Sigma_{t-1}^{-1}}\,\mhl{cone}{\mu_{t-1}}) \\
&= \mu_{t-1} + \tfrac{\sigma_{t-1}^2}{\sigma_t^2}M_t^T(x_t-M_t\mu_{t-1})
\end{aligned}
\)
\\

\bottomrule
\end{tabular}
\end{adjustbox}
\vspace{-0.2in}
\end{table*}

\section{\ours{} (\oursshort)}
\label{sec:method}
In this section, we introduce a new family of diffusion processes that use a generalized linear degradation operation for degrading the signal, in addition to the standard additive Gaussian noise. We then show how this formulation can be implemented in deep learning frameworks such as PyTorch~\cite{paszke2019pytorch} for any choice of a linear degradation that is available as a function call. In our case, we choose a downsizing operator as our linear degradation. Next, we present our training and sampling pipelines. Finally, we introduce our architecture that can handle scale-preservation and scale-changing transitions at multiple resolutions.

\subsection{Generalized Linear Diffusion Process}
\subsubsection{Extension to Linear Degradation}
We now replace the scalar coefficient of $x_{t-1}$ in Eq.~\ref{eq:forward_step_ddpm}, \ie, $\sqrt{\alpha_t}$,  with a more generic linear operator $M_t$. For example, blurring or downsampling  can serve as such a linear operator. Let us assume a Gaussian distribution for this updated formulation for the transition distribution $q(x_t|x_{t-1})$ as 
$x_t = M_t x_{t-1}+\eta_t, \eta_t \sim \mathcal{N}(0, \Sigma_{t|t-1})$. Here, we do not assume $\Sigma_{t|t-1}$ to be isotropic. 

Now, repeatedly sampling the next state using the transition distribution, we want to derive an equation analogous to Eq.~\ref{eq:make_noisy_ssd_analysis}, which provides us $x_t$ given $x_0$.
It is clear that this will also be a Gaussian distribution $q(x_t|x_0)= \mathcal{N}(\mu_t,\Sigma_t)$. The only constraint we want to enforce is isotropy, \ie, $\Sigma_t=\sigma_t^2\textbf{I}$. 
For the coefficient of $x_0$, 
 instead of $\sqrt{\Bar{\alpha}_t}=\sqrt{\alpha_t}\sqrt{\alpha_{t-1}}\dots\sqrt{\alpha_1}$ in Eq.~\ref{eq:make_noisy_ssd_analysis}, we get $M_{1:t}=M_t M_{t-1} \dots M_1$, \ie, $\mu_t=M_{1:t}x_0$.
 Hence, $q(x_t|x_0)=\mathcal{N}(M_{1:t}x_0,\sigma_t^2 \textbf{I})$, which can be expressed as:
\begin{equation}
\label{eq:make_noisy_ssd_matrix}
x_t = M_{1:t}x_0 +\sigma_t\epsilon, \quad \epsilon \sim \mathcal{N}(0, \textbf{I})
\end{equation}
Using Theorem~\ref{thm:forward_covariance_linear_diffusion}, similar to blurring diffusion~\cite{hoogeboom2022blurring}, the transition distribution $q(x_t|x_{t-1})$ is given by:
\begin{equation}
\label{eq:forward_step_ddpm_matrix}
\begin{split}
x_t &= M_{t}x_{t-1} + \eta_t, \quad \eta_t \sim \mathcal{N}(0, \Sigma_{t|t-1}), \\
\text{where } \Sigma_{t|t-1}&=\Sigma_t - M_t\Sigma_{t-1}M_t^T.
\end{split}
\end{equation}

For the isotropic marginals $\Sigma_t=\sigma_t^2\textbf{I}$ and $\Sigma_{t-1}=\sigma_{t-1}^2\textbf{I}$, we obtain 
$\Sigma_{t|t-1}= \sigma_t^2\textbf{I} - \sigma_{t-1}^2M_tM_t^T$.
For positive semi-definite feasibility we require $\sigma_t^2\textbf{I}\succeq\sigma_{t-1}^2M_tM_t^T$, \ie, $\sigma_t^2\ge\sigma_{t-1}^2\lambda_\text{max}(M_tM_t^T)$.

As shown in Theorem~\ref{thm:posterior_linear_diffusion}, the reverse diffusion step, \ie, the posterior distribution $q(x_{t-1}|x_t, x_0)$, conditioned additionally on $x_0$, is also a normal distribution:
\begin{equation}
\label{eq:reverse_step_ddpm_matrix}
\begin{split}
q(x_{t-1}|x_t, x_0) &= \mathcal{N}(\mu_{t\rightarrow t-1}, \Sigma_{\mu_{t\rightarrow t-1}}),\\
\text{where } \Sigma_{t\rightarrow t-1}&=(\Sigma_{t-1}^{-1}+M_t^T\Sigma_{t|t-1}^{-1}M_t)^{-1}, \text{ and} \\
\mu_{t\rightarrow t-1}&=\Sigma_{t\rightarrow t-1}(\Sigma_{t-1}^{-1}\mu_{t-1}+M_t^T\Sigma_{t|t-1}^{-1}x_t)
\end{split}
\end{equation}

Using the Woodbury matrix identity and isotropic covariance assumption, this simplifies to (Theorem~\ref{thm:isotropic_posterior_simplification}):
\begin{equation}
\label{eq:reverse_step_ddpm_matrix_sigma_simplified_isotropic}
\begin{split}
\Sigma_{t\rightarrow t-1} &= \sigma_{t-1}^2\textbf{I}-\frac{\sigma_{t-1}^4}{\sigma_t^2}M_t^TM_t\\
\mu_{t\rightarrow t-1} &= \mu_{t-1} + \frac{\sigma_{t-1}^2}{\sigma_t^2}M_t^T(x_t-M_t\mu_{t-1})
\end{split}
\end{equation}
Please refer to Table~\ref{tab:linear_operator_diffusion} for the comparison of our Generalized Linear Diffusion Process framework against DDPM and Blurring Diffusion (BD).

\smallskip
\noindent \textbf{DDPM as a special case of \oursshort{}.} When $M_t=\sqrt{\Bar{\alpha}_t}\mathbf{I}$ and $\sigma_t=\sqrt{1-\Bar{\alpha}_t}$, the forward, marginal, and posterior distributions of \oursshort{} collapse to those of the DDPM model. 

\subsubsection{Implementation Details}
\noindent \textbf{Choice of $\mathbf{M_t}$.} We derived the above framework so that we can introduce scale spaces from Gaussian pyramids into the diffusion process. Although $M_t$ may be any arbitrary linear operator, for our purposes we select it to be a resize operator, which effectively blurs and downsamples the image, and then multiplies it by $\text{a}_t=\sqrt{\Bar{\alpha_t}}$, as shown in Algo.~\ref{alg:implicit_linear_operator}. Note that this changes the dimensionality of the signal, in contrast with previous diffusion formulations. However, since we make no assumptions about dimensionality, our framework remains valid regardless. Furthermore, with this choice of $M_t$, we also define a resolution schedule $r(t)$ that maps diffusion timestep ($t$) to the corresponding resolution, such that the resolution monotonically decreases as $t$ increases (Fig.~\ref{fig:r-t_schedule}). Refer to the supplementary Sec.~\ref{sec:ssd_validity} for another degradation example.

\noindent \textbf{Calculating the Transpose.} Since operators like image resizing are implicit, we may not have the matrix form available, making it non-trivial to apply the transpose $M_t^T$. To address this, we use a vector-Jacobian product of the function call $M_t(\cdot)$, \ie, $M^Tv=$ \code{torch.autograd.grad(}$M_t$\code{(x), x, grad\_outputs=}$v$\code{)[0]}, which, for linear operators, does not depend on \code{x}, as shown in Algo.~\ref{alg:implicit_linear_operator_transpose}. This computes the derivative of the inner product $\langle v,M_tx \rangle$ with respect to $x$, \ie, $\nabla_{x}\langle v,M_tx \rangle=M_t^Tv$.

\noindent \textbf{Sampling from a Non-Isotropic Gaussian Distribution.} A neat trick to sample from a non-isotropic Gaussian distribution with covariance matrix $\Sigma$ is to first sample a standard Gaussian noise $\epsilon \sim \mathcal{N}(0, \textbf{I})$, and then multiply with the square root of the covariance matrix, so that $\Sigma^{\frac{1}{2}}\epsilon \sim \mathcal{N}(0, \Sigma^{\frac{1}{2}} \textbf{I} (\Sigma^{\frac{1}{2}})^T)=\mathcal{N}(0, \Sigma)$. In our case, we need to sample noise from $\Sigma_{t \rightarrow t-1}$ from Eq.~\ref{eq:reverse_step_ddpm_matrix_sigma_simplified_isotropic}, which depends on implicit operators $M_t(\cdot)$ and $M_t^T(\cdot)$. Thus, we need a way to apply $\Sigma_{t \rightarrow t-1}(\cdot)$ implicitly to a standard Gaussian noise $\epsilon$. For this purpose, we use the Lanczos algorithm~\cite{lanczos1950iteration, golub2009matrices}, which numerically computes $A(x)$ given an implicit symmetric linear operator $A(\cdot)$ and vector $x$. When the Lanczos algorithm is applied with a square root spectral function over the eigenvalues, we can obtain $A^{\frac{1}{2}}x$. In our case, this gives
$\eta_{t \rightarrow t-1}=\Sigma_{t \rightarrow t-1}^\frac{1}{2}\epsilon \sim \mathcal{N}(0, \Sigma_{t \rightarrow t-1})$ as shown in Algo.~\ref{alg:sample_non_isotropic_noise}.

\definecolor{codeblue}{rgb}{0.25,0.5,0.5}
\definecolor{codekw}{rgb}{0.85, 0.18, 0.50}

\definecolor{codesign}{RGB}{0, 0, 255}
\definecolor{codefunc}{rgb}{0.85, 0.18, 0.50}
\definecolor{codeblack}{RGB}{0,0,0}

\definecolor{codecall}{rgb}{0.5, 0.2, 0.8}

\lstdefinelanguage{PythonFuncColor}{
  language=Python,
  keywordstyle=\color{blue}\bfseries,
  commentstyle=\color{codeblue},  %
  stringstyle=\color{orange},
  showstringspaces=false,
  basicstyle=\ttfamily\small,
  literate=
    {*}{{\color{codesign}* }}{1}
    {-}{{\color{codesign}- }}{1}
    {+}{{\color{codesign}+ }}{1}
    {\*\*}{{\color{codesign}** }}{1}
    {-->}{{\color{codeblue}->}}{1}
    {-1}{{\color{codeblue}-1}}{1}
    {torch.autograd.grad}{{\color{codecall}torch.autograd.grad}}{1}
    {torch.zeros}{{\color{codecall}torch.zeros}}{1}
    {F.interpolate}{{\color{codecall}F.interpolate}}{1}
    {lanczos}{{\color{codecall}lanczos}}{1}
    {M\_input\_shape}{{\color{codeblack}M\_input\_shape}}{1}
    {shape}{{\color{codecall}shape}}{1}
}

\lstset{
  language=PythonFuncColor,
  backgroundcolor=\color{white},
  basicstyle=\fontsize{6.5pt}{7pt}\ttfamily\selectfont,
  columns=fullflexible,
  breaklines=true,
  captionpos=b,
}

\lstset{
  emph={M,M_T,sample_non_isotropic_noise, train_iter, sample_iter, diffuse, model, backward, step, zero_grad, calculate_posterior_noise, calculate_posterior_mean, cummulative_M, min_snr_5},
  emphstyle={\color{codefunc}},
}
\lstset{
  literate=
    {Eq_diffuse_matrix}{{Using Eq.~\ref{eq:make_noisy_ssd_matrix}}}1
    {Eq_denoise_matrix}{{Using Eq.~\ref{eq:reverse_step_ddpm_matrix_sigma_simplified_isotropic}}}1
    {Eq_loss_ssd}{{Using Eq.~\ref{eq:loss_ssd}}}1
}
\begin{figure}[t]
\centering
\begin{minipage}{\linewidth}
\begin{algorithm}[H]
\caption{Implicit Linear Operator}
\label{alg:implicit_linear_operator} 
\begin{lstlisting}
# M resizes and attenuates signal x
def M(x, a_t, a_t_minus1, size_out):
    return (a_t / a_t_minus1) * F.interpolate(
        x, size=size_out, mode="bilinear",
        align_corners=False, antialias=True)

\end{lstlisting}
\end{algorithm}
\end{minipage}\hfill
\begin{minipage}{\linewidth}
\begin{algorithm}[H]
\caption{Implicit Linear Operator's Transpose}
\label{alg:implicit_linear_operator_transpose}
\begin{lstlisting}
# M_T applies the transpose of M on v
def M_T(M, v, a_t, a_t_minus1, M_input_shape):
    size_out = v.shape()[-2:]
    with torch.enable_grad():  
        x = torch.zeros(M_input_shape, requires_grad=True)
        out = M(x, a_t, a_t_minus1, size_out)
        # calculate M^Tv = d<v,Mx>/dx
        (g,) = torch.autograd.grad(out, x, grad_outputs=v, retain_graph=False)
    return g
\end{lstlisting}
\end{algorithm}
\end{minipage}\hfill
\begin{minipage}{\linewidth}
\begin{algorithm}[H]
\caption{Sampling Non-Isotropic Gaussian Noise}
\label{alg:sample_non_isotropic_noise}
\begin{lstlisting}
# Sample noise from posterior covariance Sigma_{t-->t-1}
def sample_non_isotropic_noise(M, M_T, sigma_t, sigma_tminus1, x):
    rho = (sigma_tminus1 ** 2) / (sigma_t ** 2)
    # Define matvec operator A
    A = lambda v: v - rho * M_T(M(v))
    # Lanczos approximation of A^{1/2}v
    y = lanczos(A, x, f=lambda l: l.sqrt())
    return sigma_tminus1 * y
\end{lstlisting}
\end{algorithm}
\end{minipage}

\vspace{-0.2in}
\end{figure}

\subsection{Training and Sampling}

To reverse the diffusion process using Eq.~\ref{eq:reverse_step_ddpm_matrix_sigma_simplified_isotropic}, our model must predict $\mu_{t-1}=M_{1:t-1}x_0$, which, with our choice of $M_t$, reduces to a scaled version of an image $x_0$ at resolution $r(t-1)$.  
To train such a model, using our Generalized Linear Diffusion Process, we need to first adapt $\mathcal{L}_\text{simple}$. When predicting $x_0$, the loss becomes $\mathcal{L}_\text{simple}=\mathbb{E}_{x_0,t,\epsilon}[s^2(t)\|x_0^\theta(x_t, t) - x_0 \|_2^2]$, where $s^2(t)$ is the signal to noise ratio, as shown in \cite{salimans2022progressive}. In Min-SNR-$\gamma$~\cite{hang2023efficient} instead of the $s^2(t)$ weighting, they use $\min(s^2(t),\gamma)$, with $\gamma=5$, which improves the performance of $x_0$ parameterization significantly. Following this, our loss function evaluates to: 
\begin{equation}
\label{eq:loss_ssd}
\begin{split}
\mathcal{L}=\mathbb{E}_{x_0,t,\epsilon}\Big[\text{min}(s^2(t), \gamma)\Big\|x_{0,\theta}^{r(t-1)}(x_t, t) - \frac{1}{\text{a}_{t-1}}M_{1:t-1}x_0 \Big\|_2^2\Big]
\end{split}
\end{equation}
where we predict an unscaled $\mu_{t-1}$ using a neural network $x_{0,\theta}^{r(t-1)}$ (Algo.~\ref{alg:training}). Note that the input resolution $r(t)$ of $x_t$ may be smaller than the resolution of the output at $r(t-1)$ as seen in Fig.~\ref{fig:method} (left). In standard diffusion training, timesteps are simply sampled uniformly for each batch. However, this is non-trivial in our setting because the $(r(t),r(t-1))$ pairs may not match.
To solve this, we first uniformly sample a single $t$, and if $r(t)=r(t-1)$, then uniformly sample the batch size number of $t_i$'s that have $r(t_i)=r(t)$. Otherwise, if $r(t)\neq r(t-1)$, then we fill the entire batch with the same $t$, so there is no size mismatch. Since not all $t$'s change resolution, many of the $M_t$'s can be replaced by scalar multiplication with $(\text{a}_t/\text{a}_{t-1})=\sqrt{\alpha_t}$.

For sampling (Algo.~\ref{alg:sampling}), we start from a random Gaussian noise at the lowest resolution $r(T)$. Our model $x_{0,\theta}^{r(t-1)}$ predicts a clean image at the next resolution $r(t-1)$, using which we can calculate $\mu_{t-1}$ and denoise using the posterior distribution (Eq.~\ref{eq:reverse_step_ddpm_matrix_sigma_simplified_isotropic}). This also involves sampling from $\Sigma_{t\rightarrow t-1}$, which may not be isotropic, and hence we use Algo.~\ref{alg:sample_non_isotropic_noise} to sample noise from this distribution. Eq.~\ref{eq:reverse_step_ddpm_matrix_sigma_simplified_isotropic} is equivalent to DDPM sampling when $r(t)=r(t-1)$, so the non-isotropic noise sampling can be replaced with normal \code{torch.randn()} calls for resolution-preserving steps.

\begin{figure}[t]
\centering
\begin{minipage}{\linewidth}
\begin{algorithm}[H]
\caption{Train}
\label{alg:training}
\begin{lstlisting}
def train_iter(x, t, a_t_minus1, model, opt):
    opt.zero_grad()        
    t_minus1 = (t-1).clamp(min=0)
    # clean image at res r(t-1) = M_{1:t-1}(x) / a_{t-1}
    x_start_t_minus1 = cummulative_M[t_minus1](x)/a_t_minus1
    # Eq_diffuse_matrix
    x_t = diffuse(x, t) 
    pred_x_start_t_minus1 = model(x_t, t)
    # Eq_loss_ssd
    loss = min_snr_5(t) * ((pred_x_start_t_minus1 - x_start_t_minus1) ** 2)
    loss.backward()
    opt.step()
    return loss
\end{lstlisting}
\end{algorithm}
\end{minipage}
\begin{minipage}{\linewidth}
\begin{algorithm}[H]
\caption{Sampling}
\label{alg:sampling}
\begin{lstlisting}
# get x_{t-1} given x_t
def sample_iter(x_t, t, model):
    pred_x_start_t_minus1 = model(x_t, t)
    mu_t_minus1 = a_t_minus1 * pred_x_start_t_minus1
    # Eq_denoise_matrix
    posterior_noise = calculate_posterior_noise(t)
    posterior_mean = calculate_posterior_mean(x_t, mu_t_minus1, t)
    x_t_minus1 = posterior_mean + posterior_noise
    return x_t_minus1
\end{lstlisting}
\end{algorithm}
\end{minipage}
\vspace{-0.2in}
\end{figure}

\begin{figure*}
    \centering
    \includegraphics[width=0.9\linewidth]{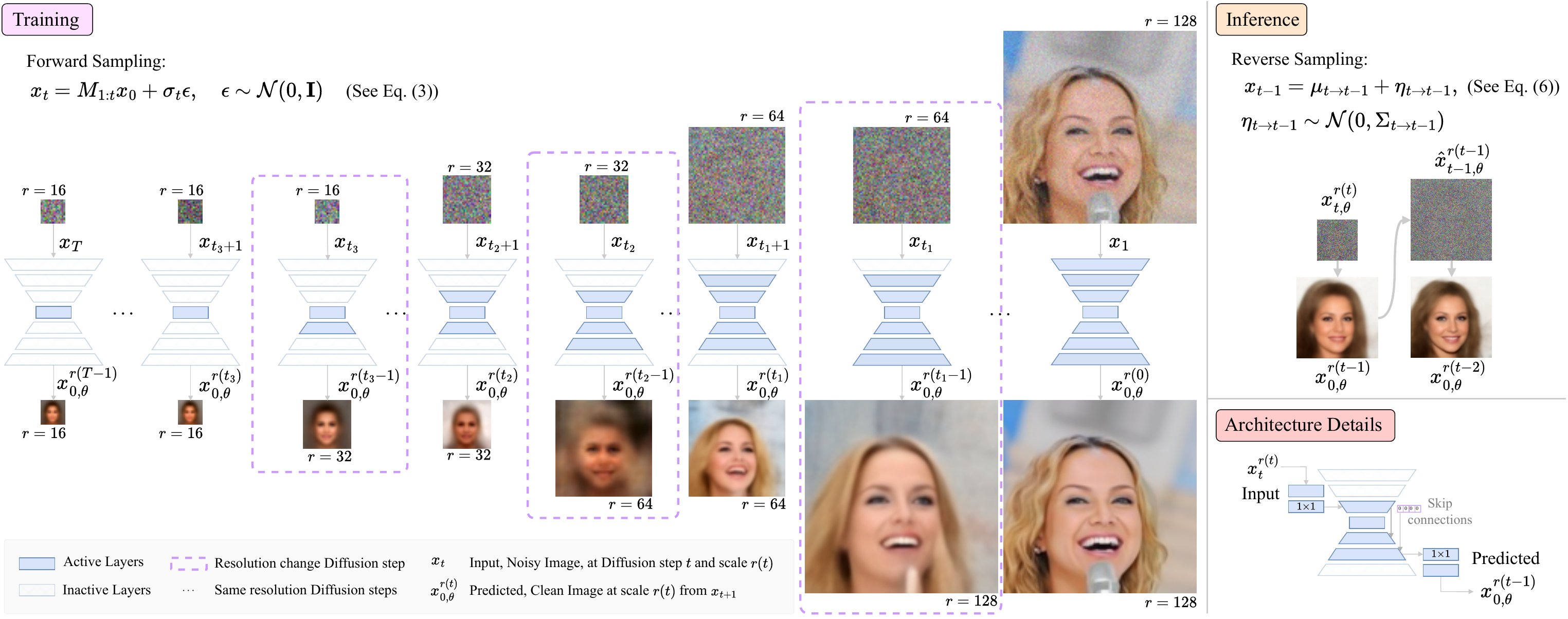}
    \caption{\textbf{Overview.} Left: During training $x_t$'s at resolution $r(t)$ are sampled using Eq.~\ref{eq:make_noisy_ssd_matrix}, and our model is trained to predict clean image $x_{0,\theta}^{r(t-1)}$ using the loss as in Eq.~\ref{eq:loss_ssd}. Our \ourmodel{} is able to process both resolution-preserving and resolution-changing steps at multiple resolution using only parts of the network. Right-top: During sampling, Eq.~\ref{eq:reverse_step_ddpm_matrix_sigma_simplified_isotropic} is used to progressively denoise and upsample to generate images. Right-bottom: Our \ourmodel{} has additional 1${\times}$1 Conv layers to take inputs at any UNet encoder block and get outputs form any decoder blocks. For resolution changing, the skip connections are fed with zero-filled tensors. }
    \label{fig:method}
    \vspace{-0.2in}
\end{figure*}

\subsection{Architecture}
\label{subsec:architecture}
We adapt the UNet architecture from Ablated Diffusion Model (ADM)~\cite{guided-diffusion} to design our proposed model FlexibleUNet (\ourmodel{}), which supports multi-resolution inputs and outputs to fully realize the scale-space formulation. 
Because \ours{} embeds a resizing operator in the forward diffusion process, the spatial resolution of $x_t$ varies across timesteps, and the reverse model must therefore operate on variable-sized noisy states and sometimes predict a higher-resolution output at the next scale (Fig.~\ref{fig:method}).

A standard diffusion model, such as ADM~\cite{guided-diffusion}, is trained to operate at a single fixed resolution throughout all timesteps, and even multiresolution UNet variants only process multiple scales \emph{within} a fixed-resolution diffusion process. In contrast, \oursshort{} requires an architecture that natively handles different input resolutions across timesteps. To address this, we explore two architectural designs.

\noindent \textbf{Full UNet (Single Path).} 
The base UNet architecture inherently supports variable-size inputs and outputs, and in principle can operate on any spatial resolution $R{\times}R$ as long as the kernel sizes, strides, padding, and pooling operations produce valid feature maps at every layer. However, this design has two key limitations for \ours{}. First, it requires the input and output resolutions to be equal. In our setting, certain timesteps involve a resolution transition, which would require the model to output at a higher resolution. To handle this, the input must be manually upsampled before entering the UNet whenever such a transition occurs.

Second, the depth of the UNet determines how many distinct spatial scales it can represent. For a UNet with $L$ downsampling blocks, the smallest internal resolution is $\frac{R}{2^{L-1}}$, which fixes the total number of scales to $L$. This number is typically small and does not grow with the input resolution. For example, the ADM architecture uses 4 feature map resolutions for $64{\times}64$, 5 for $128{\times}128$, and 6 for $256{\times}256$, meaning that across all these models the number of downsampling stages remains fixed at 4. Thus, even at higher resolutions, the network cannot represent more than a handful of scales, limiting the usefulness of a scale-space formulation where many more levels naturally exist.

\noindent \textbf{\ourmodel{}.} 
These limitations motivate our proposed architecture, \ourmodel{}, where different subsets of UNet layers are dynamically activated based on the input resolution. High-resolution inputs traverse the full UNet, while lower-resolution inputs are routed only through the deeper layers, effectively bypassing the early and late blocks. Since each block expects a specific channel dimensionality, we insert $1{\times}1$ conv layers to map the input features to the appropriate channel size while preserving spatial resolution.

For denoising steps that do not involve a resolution change, the active pathway through the UNet remains symmetric, using the same number of downsampling and upsampling blocks. When a resolution increase is required, the pathway becomes asymmetric: the model uses one additional upsampling block compared to the number of downsampling blocks encountered. In these cases, the skip connections that would normally come from the bypassed encoder blocks are replaced with zero tensors (Fig.~\ref{fig:method}). This design allows the model to share parameters across resolutions while supporting valid diffusion dynamics during resolution transitions.

\addtocontents{toc}{\protect\setcounter{tocdepth}{2}} %

\addtocontents{toc}{\protect\setcounter{tocdepth}{-1}}

\section{Experiments}
\label{sec:experiments}
\noindent \textbf{Datasets. } We perform experiments and analyze the performance of \ours{} on the CelebA dataset\cite{liu2015faceattributes} and the ImageNet dataset~\cite{deng2009imagenet}. The CelebA dataset consists of around 200K training images, while the ImageNet dataset contains around 1.3 million images from 1000 different classes. We use JPEG images for these datasets. We conduct experiments at $64{\times}64$ resolution for both CelebA and ImageNet. We additionally show experiments on $128{\times}128$ and $256{\times}256$ for CelebA dataset. The CelebA experiments helps us understand our method's scalability with increasing resolutions, while ImageNet helps in evaluating the model's ability to learn complex and diverse distributions. 
\begin{figure}[t]
    \centering
    \vspace{-0.5em}
    \includegraphics[width=\linewidth]{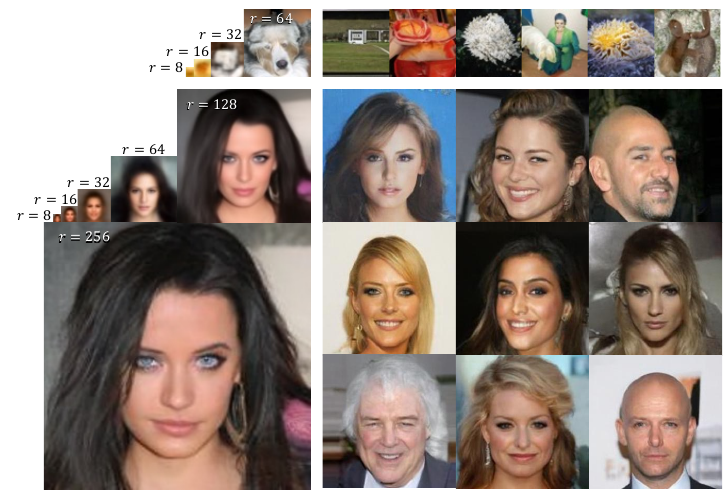}
    \caption{\textbf{Visual samples.} Top: ImageNet-64 unconditional generation. For the top-most sample we also show model prediction at various scales (8, 16, 32, 64) during \oursshort. Bottom: CelebA-256 unconditional generation. For the top-most sample we also show model predictions at various scales (8, 16, 32, 64, 128, 256).}
    \label{fig:viz}
    \vspace{-0.2in}
    \end{figure}
    
\noindent \textbf{Unconditional Image Generation. } We analyze and evaluate \ours{} on unconditional image generation, as it allows us to clearly study how scale-space theory integrates with the diffusion process.  

\noindent \textbf{Implementation Details. }
We use the ADM~\cite{guided-diffusion} repository as our base codebase and build our baselines (DDPM~\cite{ddpm}, Blurring Diffusion~\cite{hoogeboom2022blurring}), as well as \ours{} on top of it. For DDPM, we consider two standard parametrizations as our baselines, the $\epsilon$-prediction, and the $x_0$-prediction formulation. We train the baseline model with Min-SNR-$\gamma$ weighting for the $x_0$-parametrization to ensure an accurate comparison to \ours.  We implement Blurring Diffusion using the pseudo-code provided in their paper. 

\begin{table}[!t]
\begin{minipage}[t]{\linewidth}
\centering
\caption{\textbf{Main Results.} Unconditional image generation results on CelebA dataset over multiple resolutions. Training time is specified in hours. Average GFlops per iteration. Effective batch size is 128 for resolutions 64 and 128, 64 for resolution 256. Here BD refers to Blurring Diffusion~\cite{hoogeboom2022blurring} and all SSD models use our~\ourmodel{}~architecture. }\vspace{-0.1in}
\resizebox{1\textwidth}{!}{
    \label{tab:celeba_main_table}
    
    \begin{tabular}{@{}lccccccccc@{}}
        \toprule
        \multirow{2}{*}{Method} &  \multicolumn{3}{c}{CelebA-64 (1M iters)} &  \multicolumn{3}{c}{CelebA-128 (300K iters)} &  \multicolumn{3}{c}{CelebA-256 (300K iters)} \\
        \cmidrule[\cmidrulewidth](l){2-4} \cmidrule[\cmidrulewidth](l){5-7} \cmidrule[\cmidrulewidth](l){8-10}
         & FID & Time & GFlops  & FID & Time & GFlops & FID & Time & GFlops \\
         \midrule
        DDPM-$\epsilon$    & 2.22 & 70.30 & 60.05 &  4.16 & 50.50 & 132.30 &  5.52 & 87.31 & 497.03 \\
        DDPM-$x_0$         & 2.98 & 70.71 & --    &  3.50 & 50.33 & --     &  5.47 & 87.33 & -- \\
        BD & 2.06 & 71.79 & --    &  3.67 & --    & --     &  4.76 & 88.08 & -- \\
        \midrule
        \oursshort~(2L) &  2.14 & 62.63 & 50.61 &  --    & --    & --    & --    & --    & --    \\
        \oursshort~(3L) &  3.61 & 56.13 & 44.27 &  6.53  & 31.71 & 87.38 & 7.79  & 59.00 & 317.36 \\
        \oursshort~(4L) &  4.28 & 52.38 & 38.48 &  --    & --    & --    & 10.52 & 51.70 & 272.98\\
        \oursshort~(5L) &  --   & --    & --    &  10.47 & 25.41 & 66.72 &  --   & --    & 237.70 \\
        \oursshort~(6L) &  --   & --    & --    &  --    & --    & --    & 13.50 & 42.88 & 209.69\\
        \bottomrule
    \end{tabular}
    
    }
\end{minipage}
\vspace{-0.2in}
\end{table}

\begin{figure}[t]
\centering
\resizebox{0.9\linewidth}{!}{
\begin{minipage}{0.57\linewidth}
    \centering
\centering
    \includegraphics[width=\linewidth]{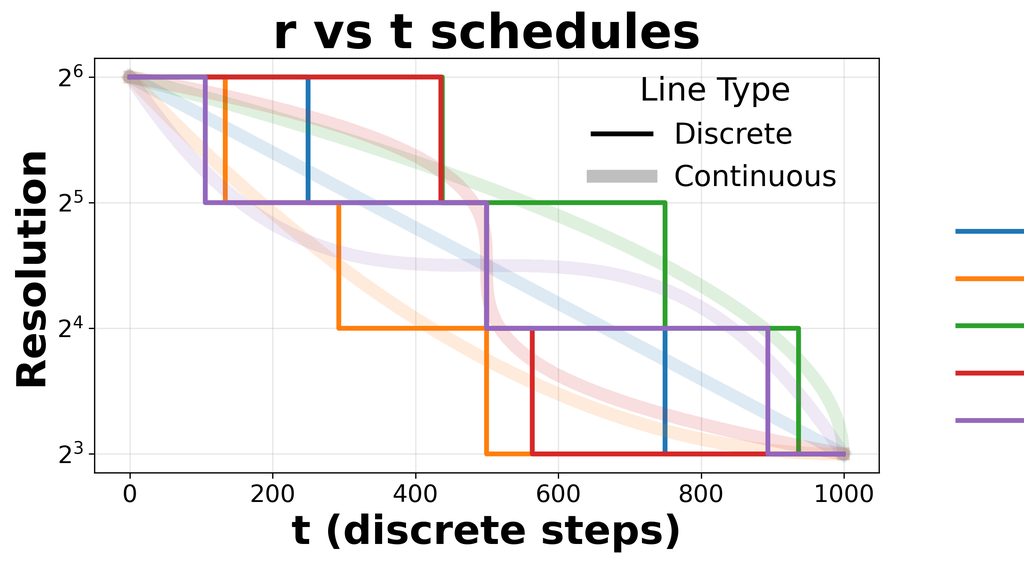}
\end{minipage}
\begin{minipage}{0.41\linewidth}
    \centering
    \vspace{2.2pt}
    \resizebox{\linewidth}{!}{%
        \begin{tabular}{@{}lcc@{}}
            \toprule
            Schedule & FID & Time (hrs) \\
            \midrule
            equal               & 9.64          & 12.88 \\
            ConvexDecay\_2      & 11.03         & \textbf{11.71} \\
            ConvexDecay\_0.5    & \textbf{4.87} & 13.81 \\
            SigmoidLikeDecay\_3 & 7.08          & 13.06 \\
            TanhLikeDecay\_3    & 8.09          & 12.50 \\
            \bottomrule
        \end{tabular}
    }
\end{minipage}
}
\vspace{-0.1in}
    \captionof{figure}{\textbf{Resolution Schedules.} Mapping diffusion timesteps $t$ to resolution $r$ across 4 scales. Both discrete and continuous variants are shown. The right shows FIDs at 500k iterations (batch size 8).}
    \label{fig:r-t_schedule}
\vspace{-0.3in}
\end{figure}

\begin{figure}[h]
\centering
\begin{minipage}{0.35\linewidth}
    \centering
    \captionof{table}{ \textbf{ImageNet-64 Results.} Unconditional image generation results on ImageNet-64 dataset.}
    \centering
    \resizebox{0.9\linewidth}{!}{
    \label{tab:imagenet64_main_table}
    \begin{tabular}{@{}lc@{}}
    \toprule
    Method & FID \\
     \midrule
    DDPM-$\epsilon$ & 12.82 \\
    DDPM-$x_0$ & 13.07 \\
    Blurring Diffusion & 15.34   \\
    \midrule
    \oursshort~(2L) & 13.08   \\
    \oursshort~(4L) & 17.89   \\
    \bottomrule
    \end{tabular}
    }
\end{minipage}
\hfill
\begin{minipage}{0.6\linewidth}
\vspace{0.12in}
    \centering
     \captionof{figure}{\textbf{Temporal Scaling.} Training time of proposed \oursshort~with our \architecturename{}~ across multiple resolutions.}
    \vspace{-0.1in}
    \label{fig:training_time_across_resolutions}
    \includegraphics[width=0.85\linewidth]{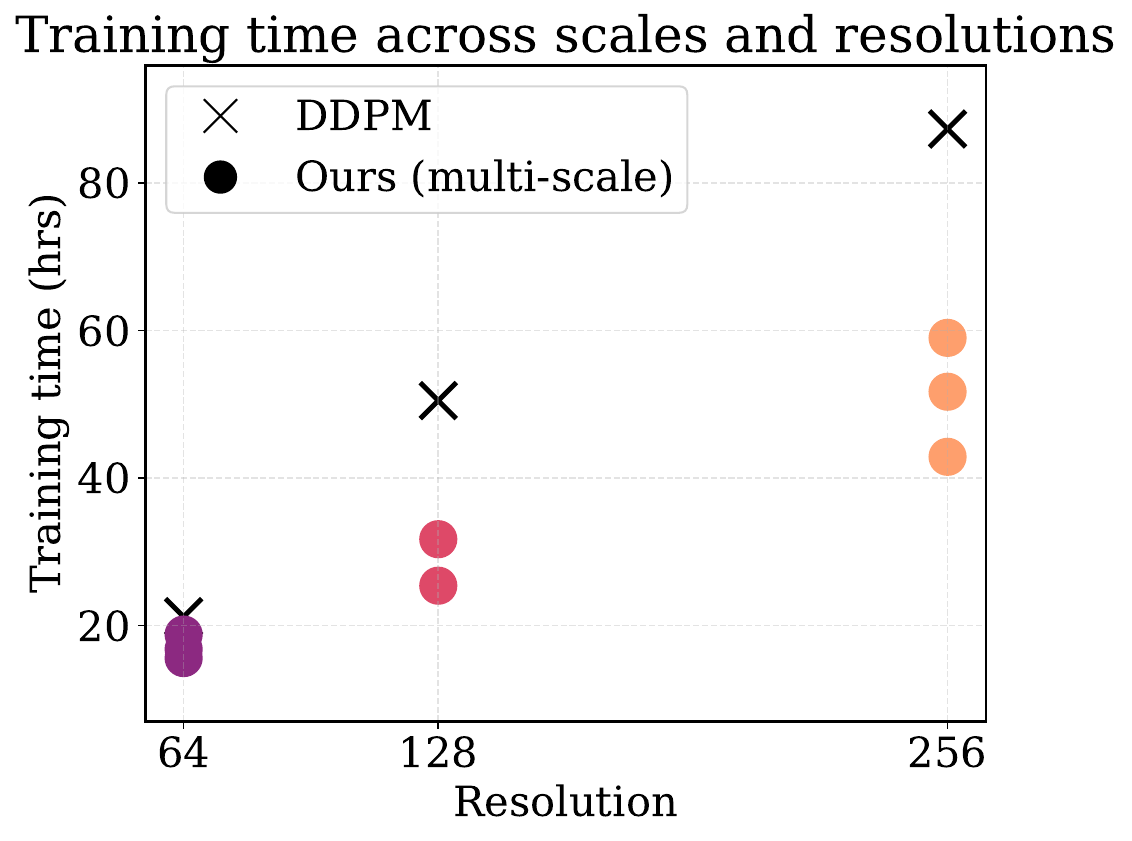}
\end{minipage}
\vspace{-0.3in}
\end{figure}

For the diffusion process, we follow the linear noise schedule proposed in DDPM~\cite{ddpm} and use the standard setting of 1000 timesteps. For training, we use AdamW~\cite{kingma2014adam, loshchilov2017decoupled} optimizer with a fixed learning rate. We conducted all our experiments on NVIDIA H100 and NVIDIA RTX A4000 GPUs. We maintained consistent combinations of learning rate and batch-size across dataset and resolutions. For $64{\times}64$ and $128{\times}128$, we used an effective batch size of 128, and trained the models either on a single H100, or on $4$ A4000 GPUs, with a per-GPU batch-size of 32. For $256{\times} 256$, we used an effective batch size of 64 due to memory constraints and trained them on $2$ H100s with a per-GPU batch size of 32. Our learning rate was set to $1{\times}10^{-4}$ for the $64{\times}64$ and $128{\times}128$ models, and $5{\times}10^{-5}$ for $256{\times}256$ model, following linear learning rate scaling. 

\noindent \textbf{Evaluation. }
We evaluate our models using the exponential moving average (EMA) weights with a decay rate of 0.9999. We assess the quality of generated images by computing FID~\cite{fid} scores on 50k samples \wrt the training set. We further compare \ours~with the DDPM baseline model in terms of training time, and FLOPs (Floating Point Operations) per forward pass. In addition to FLOPs, we report the sampling latency as the total time to generate a single image. 
All speed and compute metrics are measured on a single NVIDIA GH200 node.

\noindent \textbf{Main results. }
Our results are presented in Table~\ref{tab:celeba_main_table} and~\ref{tab:imagenet64_main_table}. We train the baseline models and \ours~model for 1 million iterations for CelebA-64 and 300k iterations for CelebA-128 and CelebA-256. We report the total training time, average GFLOPs per iteration, and the FID value. We notice that increasing the number of levels significantly decreases training time and GFLOPs. \oursshort~(6L) at 256 resolution takes less than half the time as the baseline DDPM. 
Table~\ref{tab:imagenet64_main_table} shows that ~\oursshort, trained for 1 million iterations, achieves comparable performance to baselines even on the harder ImageNet-64 benchmark.
Figure~\ref{fig:training_time_across_resolutions} shows that training time for \oursshort{} scales well with increasing resolution. Please refer to the supplementary Sec.~\ref{sec:more_comparisons} for more comparisons.

\noindent \textbf{Qualitative results. } We present qualitative results of our method on ImageNet-64 with \oursshort~(4L) and CelebA-256 with \oursshort~(6L) in Fig.~\ref{fig:viz}. We also show multiple intermediate predictions of the model in ~\oursshort.

\noindent \textbf{Individual effectiveness of our mathematical formulation vs architecture.} In the supplementary Sec.~\ref{sec:parts_of_approach}, we first show that our Generalized Linear Diffusion Process works, albeit suboptimally, even without \ourmodel, and also with an alternate degradation. Approximating anisotropic gaussian noise with isotropic leads to saturation artifacts, showing the need for our anisotropic sampling. Secondly, we show that \ourmodel{} is effective on its own for other formulations of iterative multi-resolution pixel-space generation, albeit suboptimal. We do this by applying it to approximate multi-resolution diffusion as in PyramidalFlow~\cite{jin2024pyramidal, chen2025pixelflow, jeong2025upsample}.

\noindent \textbf{Resolution Schedule.} $r(t)$ specifies the spatial resolution as a function of diffusion timestep. We present 5 different resolution schedules in Figure~\ref{fig:r-t_schedule} (left) at 4 levels  $(64, 32, 16, 8)$ 
(refer to the supplementary Sec.~\ref{sec:res_sch} for details) and note their effects for a CelebA-64 model in Fig.~\ref{fig:r-t_schedule} (right). We observe that schedules that spend the least number of timesteps at the higher resolutions train the fastest, but also yield the worst FID (\ie ConvexDecay\_2). In contrast, the model trained with ConvexDecay\_0.5, which spends the most steps at the highest resolution, achieves the best FID, but requires the longest training time. We use this for all our experiments.

\noindent \textbf{Full UNet vs \ourmodel{}.} In Table~\ref{tab:ablation-architecture}, we observe that \ourmodel{} has slightly better FID for both 2L and 4L, while being faster than the Full UNet. Hence, we use \ourmodel{}.

\noindent \textbf{Sampling.} Use of Lanczos has negligible overhead. Further, \oursshort~does not suffer from performance drop, like DDPM, on sampling steps reduction. (Refer supplementary Sec.~\ref{sec:quant_results}.) 

\begin{table}
  \caption{\textbf{Architecture Ablation.} FID (at 500K iterations on CelebA-64) and Inference time (in secs/generation, 1000 steps, batch size=1, on 1{}${\times}${} A4000) of network architecture variants at 2 and 4 levels.}
  \vspace{-1em}
  \label{tab:ablation-architecture}
  \centering
\resizebox{0.5\linewidth}{!}{
    \begin{tabular}{@{}lccc@{}}
    \toprule
    \multirow{2}{*}{Method} & \multicolumn{1}{c}{FID} & \multicolumn{2}{c}{Inference time} \\
    \cmidrule[\cmidrulewidth](l){2-2}
    \cmidrule[\cmidrulewidth](l){3-4}
     & res. 64 & res.64 & res. 256 \\ 
    \midrule
    Full Unet, 2L   & 2.33          & 16.19 & 43.07\\
    \ourmodel{}, 2L & \textbf{2.26} & \textbf{15.38} & \textbf{38.99}\\
    \midrule
    Full Unet, 4L   & 4.90          & 16.28 & 34.74\\
    \ourmodel{}, 4L & \textbf{4.87} & \textbf{13.43} & \textbf{31.08}\\
    \bottomrule
  \end{tabular}
  }
    \vspace{-2em}
\end{table}

\addtocontents{toc}{\protect\setcounter{tocdepth}{2}} %

\noindent \textbf{Conclusion.} We showed that diffusion models and scale spaces share an information hierarchy, and we quantified this connection mathematically. Observing that highly noised diffusion states contain only low-resolution information, we introduced a generalized family of diffusion models that embeds scale-space structure into the forward process, yielding \ours. To realize this in practice, we proposed the \ourmodel{} architecture and demonstrated its effectiveness on unconditional image generation.

\noindent \textbf{Acknowledgment.} This research is based upon work supported by the Office of the Director of National Intelligence (ODNI), Intelligence Advanced Research Projects Activity (IARPA), via IARPA R\&D Contract No. 140D0423C0076. The views and conclusions contained herein are those of the authors and should not be interpreted as necessarily representing the official policies or endorsements, either expressed or implied, of the ODNI, IARPA, or the U.S. Government. The U.S. Government is authorized to reproduce and distribute reprints for Governmental purposes notwithstanding any copyright annotation thereon.

\clearpage
\setcounter{page}{1}
\maketitlesupplementary

\begin{figure}[t]
    \centering
     \animategraphics[width=0.5\textwidth, loop, autoplay]{8}%
    {figures/supp/pred_x0/frame_}%
    {0}%
    {99}%
    \caption{Animation of the predicted clean image $x_{0,\theta}^{r(t-1)}$ over the generation process for gradual downsizing degradation operator in \oursshort~framework. (Best viewed in Adobe Reader).}
    \label{fig:gradual_downsizing_x0_animated}
\end{figure}

\tableofcontents

\section{Clarifications}

We add some clarifications for our main paper here.

\begin{enumerate}
    \item All variable names used in Algo.~\ref{alg:implicit_linear_operator}, \ref{alg:implicit_linear_operator_transpose}, and \ref{alg:sample_non_isotropic_noise} have their usual meanings. For example, \code{a\_t}, \code{a\_t\_minus1} denote the signal coefficients $\text{a}_t$, $\text{a}_{t-1}$ respectively, while \code{sigma\_t}, \code{sigma\_tminus1} denote the noise schedule $\sigma_t$, $\sigma_{t-1}$ respectively. \code{size\_out} denotes the (height, width) of the output of \code{M} operator.
\end{enumerate}

 \begin{figure}[t]
    \centering
     \animategraphics[width=0.5\textwidth, loop, autoplay]{8}%
    {figures/supp/x_t/frame_}%
    {0}%
    {100}%
    \caption{Animation of the noisy intermediate state $x_t$ over the generation process for the gradual downsizing degradation operator in \oursshort~framework. (Best viewed in Adobe Reader).}
    \label{fig:gradual_downsizing_xt_animated}
\end{figure}

\section{Future Works}

The focus of this work has been to analyze the connection between diffusion models and scale space theory, while proposing to merge them using \ours~with \ourmodel{}.
We do not use any advanced techniques to tune our framework or architectures for the most optimal performance. Instead, we use the standard hyperparameters from the base codebase to keep the choices simple and the number of experiments under check given the expense of each training. The use of advanced techniques is out of scope for this work given the conference length manuscript. 

However, there are multiple future exploration directions which have high potential for improvement in performance.
For example, adapting newer diffusion samplers instead of using DDPM-style samplers can improve both performance and inference speeds. Similarly, progressive curriculum learning for different layers or resolutions, as done by works with multi-resolution trainings~\cite{karras2017progressivegan, gu2023matryoshka}, should also yield improvement in training optimization. 

\noindent \textbf{Why not use a Transformer-based architecture?} 
The two most popularly used architectures in diffusion are -- convolutional UNet~\cite{unet} based ADM~\cite{guided-diffusion}, and vision transformer (ViT) based DiT~\cite{peebles2023dit}. 
Another popular architecture is U-ViT~\cite{bao2023all}, that combines the skip connections from UNet with a ViT architecture.
One thing to note is that, DiT was designed for latent spaces and hence did not take into consideration the blowing up of the quadratic complexity of the attention mechanism when applied in the pixel space~\cite{chen2025pixelflow}. U-ViT acknowledges this issue, and explicitly works in a latent space for higher resolutions.  
Newer works like HDiT~\cite{crowson2024scalable} try to mitigate this issue using neighborhood attention instead of global attention is all layers. 
But such non-trivial design decisions in the architecture can develop into confounding factors. 
Since our goal is to understand how scale-spaces can be integrated into diffusion models, for simplicity we stick to the standard ADM base architecture, a widely used pixel and latent diffusion architecture~\cite{rombach2022ldm}.
Nonetheless, for future work, a similar integration of scale-space theory should also be explored with transformer based architectures.

\section{Additional Material}

\subsection{Hyperparameters}
The set of hyperparameters that we use for each dataset is summarized in Table~\ref{tab:hyperparameters}. We also note additional experimental details in Table~\ref{tab:implementation_details}. 

\begin{table}[]
    \centering
    \resizebox{0.9\linewidth}{!}{
    \begin{tabular}{lcccc}
        \toprule
        Hyperparameter & CelebA-64 & CelebA-128 & CelebA-256 & ImageNet-64 \\
        \midrule
        Noise Schedule & Linear & Linear & Linear & Linear \\
        Denoising Steps & 1000 & 1000 & 1000 & 1000 \\
        Optimizer & AdamW & AdamW & AdamW & AdamW \\
        Batch Size & 128 & 128 & 64 & 128 \\
        Learning Rate & 0.0001 & 0.0001 & 0.00005 & 0.0001 \\
        Number of Iterations & 1 million & 300k & 300k & 1 million \\
        \bottomrule
    \end{tabular}
    }
    \caption{Hyperparameters for all datasets.}
    \label{tab:hyperparameters}
\end{table}

\begin{table}[]
    \centering
    \resizebox{\linewidth}{!}{
    \begin{tabular}{lcccc}
        \toprule
        Implementation Choice & DDPM-$\epsilon$ & DDPM-$x_0$ & Blurring Diffusion & \oursshort{} \\
        \midrule
        Reverse Process Variance & fixed-large & fixed-large & fixed-small & fixed-small \\
        Loss & $L_\text{simple}$ & $L_\text{simple} + \text{Min-SNR-5}$ & $L_\text{simple}$ & $L_\text{simple} + \text{Min-SNR-5}$ \\
        \bottomrule
    \end{tabular}
    }
    \caption{Additional implementation details.}
    \label{tab:implementation_details}
\end{table}
\begin{table}
  \caption{Inference time. By default we use DDPM sampling, but we also show 25$^{\dagger}$ steps DDIM~\cite{song2020ddim} speeds.}
  \label{tab:infer_time}
  \centering
  \resizebox{0.7\linewidth}{!}{
  \begin{tabular}{@{}lccc@{}}
    \toprule
    Method & \#Steps & Speedup (Inference Time) & FID \\
    \midrule
    \multirow{3}{*}{DDPM-$x_0$} & 1000 & 1.00 $\times$  & 2.98 \\
     &  250 & 4.18 $\times$  & 14.00 \\
     & 25$^{\dagger}$ & 38.87 $\times$  & 4.70 \\
    \midrule 
    \multirow{3}{*}{DDPM-$\epsilon$} & 1000 & 1.05 $\times$  & 2.22 \\
     &  250 & 4.18 $\times$  & 11.02 \\
     & 25$^{\dagger}$ & 38.06 $\times$  & 3.76 \\
     \midrule 
    \multirow{2}{*}{\oursshort (\ourmodel{}, 2L)} & 1000 & 1.18 $\times$  & 2.14 \\
    & 250 & 4.80 $\times$  & 2.87\\
    \midrule 
    \multirow{2}{*}{\oursshort (\ourmodel{}, 4L)} & 1000 & 1.58 $\times$ & 4.28 \\
    & 250 & 5.91 $\times$  & 4.90\\
    \bottomrule
  \end{tabular}
  }
\end{table}

\begin{table}
    \begin{minipage}{\linewidth}
      \caption{\footnotesize Inference time (in secs) per gen at 64 res (1000 steps, bs=1, 1 A4000): Lanczos sampling vs. torch.randn call.}
\label{tab:lanczos}
      \centering
      \resizebox{0.5\textwidth}{!}{
      \setlength\tabcolsep{4pt}
      \begin{tabular}{@{}lcc@{}} 
        \toprule
        Method & SSD (2L) & SSD (4L)\\
        \midrule
        w/ Lanczos & 15.38 & 13.43\\
        w/o Lanczos & 15.35 & 13.40 \\ 
        \bottomrule
    \end{tabular}
    }
    \end{minipage}%
    \hfill
    \begin{minipage}{.01\linewidth}
    \end{minipage}
    \vspace{-1em}
\end{table}

\subsection{Parts of our Approach}
\label{sec:parts_of_approach}

Our approach consists of two parts. The first part is the \ours~mathematical formulation and the second part is the \ourmodel{} architecture. In the main paper, we have presented the combination of both parts as our complete approach. But here we also want to show that each part is effective on its own. So, in Section~\ref{sec:ssd_validity}, we explore whether the mathematics behind \oursshort~can be applied without a modified architecture, while in Section~\ref{sec:our_arch_validity}, we check if \ourmodel{} can be used without our mathematical framework, summarized in Table~\ref{tab:parts_of_our_approach}.
\begin{table}[h]
\caption{ Parts of our approach, and validity of each part.}
\centering
\resizebox{0.5\linewidth}{!}{
\label{tab:parts_of_our_approach}
\begin{tabular}{@{}lcc@{}}
\toprule
 & \oursshort & \ourmodel{} \\
 \midrule
Section~\ref{sec:ssd_validity} & \ding{51} & \ding{55} \\
Section~\ref{sec:our_arch_validity} & \ding{55} & \ding{51} \\
Main paper& \ding{51} & \ding{51} \\
\bottomrule
\end{tabular}
}
\end{table}

\subsubsection{Validity of \oursshort}
\label{sec:ssd_validity}
One way to verify whether \oursshort~ framework works without using a modified architecture is to assume that the actual states of the diffusion model are at a certain resolution, but when passing through the model, we resize them to the model input size. Similarly, the model outputs are resized to the required output resolution before applying losses. We test this with CelebA-32 dataset just to check the correctness of \oursshort. For this, we use a DDPM reimplementation (not ours) optimized for resolution 32 images  \cite{weilai_futurexiangdiffusion_2025}, since ADM's codebase does not support that resolution, and a smaller resolution is faster to verify on. We train these models for 300 epochs and use 5 steps of resolutions (2, 4, 8, 16, 32). We note their FIDs in Table~\ref{tab:celeba32_onlymath}.

\begin{table}
\caption{Results of only \oursshort~(w/o \ourmodel{}) on CelebA-32. (Here we resize the inputs to the model input resolution.)}
\centering
\resizebox{0.8\linewidth}{!}{
\label{tab:celeba32_onlymath}
\begin{tabular}{@{}lc@{}}
\toprule
Method & FID \\
 \midrule
DDPM-$\epsilon$ & 2.85 \\
\oursshort~(w/o \ourmodel{}, 5L) & 5.55   \\ %
\oursshort~(w/o \ourmodel{}, gradual downsizing) & 4.10   \\ %
\bottomrule
\end{tabular}
}
\end{table}

\noindent \textbf{Alternative Degradation.}
All the degradations used in this work till now have been 2${\times}$ downsampling. However, given the general nature of the theory, it is not limited to just this choice. Here we test using a gradual downsizing instead of 2${\times}$ downsizing steps. In this degradation, whenever the resolution changes, it does so by only 1 pixel at a time. We try going from $2\rightarrow32$. We report its FID in Table~\ref{tab:celeba32_onlymath}. We show some static visual results in Fig.~\ref{fig:gradual_downsizing}. We show some interesting animated visualizations (view in Adobe Reader) in Fig.~\ref{fig:gradual_downsizing_x0_animated}, and Fig.~\ref{fig:gradual_downsizing_xt_animated}.

\begin{figure*}[t]
    \centering
    \begin{subfigure}{0.33\linewidth}
        \centering
        \includegraphics[width=\linewidth]{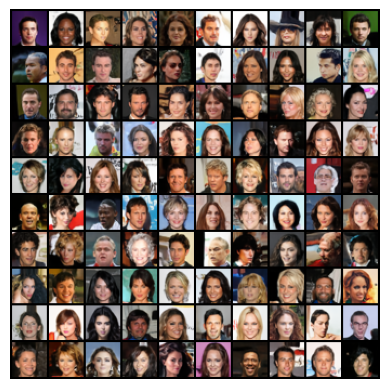}
        \caption{Generated Samples}
    \end{subfigure}
    \hfill
    \begin{subfigure}{0.33\linewidth}
        \centering
        \includegraphics[width=\linewidth]{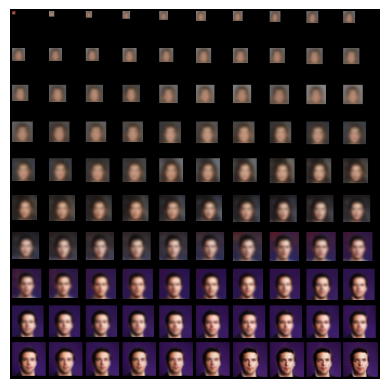}
        \caption{Predicted clean images $x_{0,\theta}^{r(t-1)}$}
    \end{subfigure}
    \hfill
    \begin{subfigure}{0.33\linewidth}
        \centering
        \includegraphics[width=\linewidth]{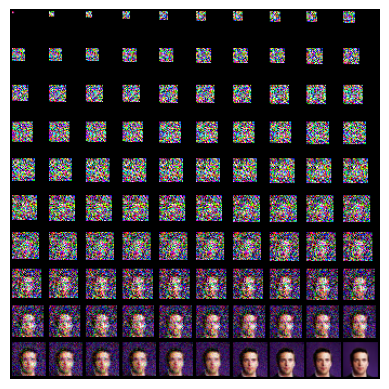}
        \caption{Noisy states $x_t$}
    \end{subfigure}

    \caption{Visual results of \oursshort~ with gradual downsizing degradation (1 pixel downsizing instead of 2${\times}$ downsizing)}
    \label{fig:gradual_downsizing}
\end{figure*}

\noindent \textbf{Effect of Isotropic Approximation.}
Another thing we wanted to test was whether we could approximate the non-isotropic Gaussian noise sampling (Algo.~\ref{alg:sample_non_isotropic_noise}) with isotropic Gaussian noise. For testing purposes, during the generation procedure (of the gradual downsizing degradation case), in the resolution changing steps, we first use Algo.~\ref{alg:sample_non_isotropic_noise} to sample non-isotropic noise, and then find the mean and variance over the height and width dimensions of this noise tensor. Instead of using the sampled non-isotropic noise for the stochasticity in Eq.~\ref{eq:reverse_step_ddpm_matrix_sigma_simplified_isotropic}, we instead use an isotropic noise sampled using \code{torch.randn()} with the calculated mean and variance. As seen in Fig.~\ref{fig:iso_approx}, this leads to the colors becoming flat and saturated, despite having facial structures. This shows that the assumption of isotropic covariance for the reverse process may not actually be valid, as assumed in ~\cite{abu2023udpm}. And we need to sample from non-isotropic Gaussians depending upon the linear operator.
\begin{figure}[h!]
    \centering
    \includegraphics[width=0.7\linewidth]{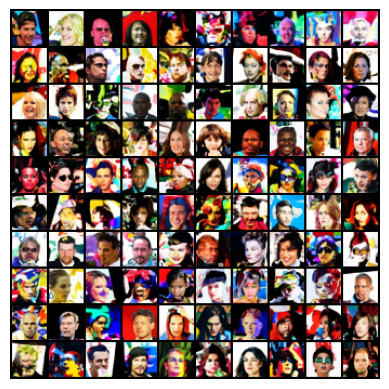}
    \caption{Effect of using isotropic noise instead of non-isotropic noise in the reverse diffusion process of \oursshort.}
    \label{fig:iso_approx}
\end{figure}

\subsubsection{Effectiveness of \ourmodel{}}
\label{sec:our_arch_validity}

In this section, we demonstrate that \ourmodel{} can naturally accommodate different formulations of the diffusion process to support multi-resolution inputs and outputs. To do so, we build upon previous works that introduce corrective noise when an upsampling operation is performed in the diffusion process~\cite{jin2024pyramidal, chen2025pixelflow, jeong2025upsample}. We implement these ideas within \ourmodel{} to both validate the flexibility of our architecture and quantify the computational benefits obtained from operating across resolutions.
A key challenge addressed in these works is the distribution mismatch that arises when a noisy latent is upsampled. Prior works~\cite{jeong2025upsample} show that applying a $2\times$ nearest-neighbor upsampling step produces a block-structured covariance that deviates from the forward diffusion trajectory. Their solution injects structured noise and identifies an adjusted timestep that realigns the upsampled latent with the original process. While this motivates our analysis, our setting is different from this in two ways: a) we operate entirely in pixel-space rather than latent space, and b) we consider multiple (more than one) upsampling stages throughout the denoising process. With these conditions in mind, our setup is as follows:

Let $x_t^r$ be a valid DDPM forward state at a timestep t for resolution r: 
\begin{equation*}
\begin{aligned}
    x_t^r &\sim \mathcal{N}\!\left(\sqrt{\bar{\alpha}_t}\, x_0^r,\; (1-\bar{\alpha}_t) I \right), \\
    \text{Let } x_t^R &= \mathrm{Upsample}(x_t^r),  \\
    x_t^R &\sim \mathcal{N}\!\left(\sqrt{\bar{\alpha}_t}\, U x_0^r,\; (1-\bar{\alpha}_t)\, U U^{\top}\right) \\
    &\text{ where } U \text{ is the Upsampling matrix}. 
\end{aligned}
\end{equation*}

Let $x_s^R$ be a valid DDPM forward state at some other timestep $s$ for resolution $R$:
\begin{equation*}
    x_s^R \sim 
    \mathcal{N}\!\left(
        \sqrt{\bar{\alpha}_s}\, x_0^R,\;
        (1-\bar{\alpha}_s)\, I
    \right),
    \qquad 
\end{equation*}

The upsampled state $x_t^R$ has covariance proportional to $UU^{\top}$, which differs from the isotropic Gaussian noise assumed by the DDPM forward process at resolution $R$. To correct this mismatch, we add corrective noise and roll back to a previous timestep. Let the corrected sample be
\[
    \tilde{x}_t^{R} = a\, x_t^R + b\, z,\qquad 
    z \sim \mathcal{N}\!\left(0,\; I - c\, U U^{\top}\right).
\]

Then the distribution of $\tilde{x}_t^{R}$ is
\begin{equation*}
\tilde{x}_t^{R} 
\sim 
\mathcal{N}\!\left(
    a \sqrt{\bar{\alpha}_t}\, U x_0^r,\;
    a^2 (1-\bar{\alpha}_t)\, U U^{\top}
    + 
    b^2 \left(I - c\, U U^{\top}\right)
\right).
\end{equation*}

We make an approximation to match the mean and covariance of $\tilde{x}_t^{R}$ to $x_s^R$
\begin{equation}
\label{eq:distribution_matching}
\begin{split}
    a^2 \bar{\alpha}_t = \bar{\alpha}_s \\
    b^2 = 1 - \bar{\alpha}_s \\
    a^2 (1-\bar{\alpha}_t) = b^2 c
\end{split}
\end{equation}

Solving the three equations mentioned in Equation~\ref{eq:distribution_matching}, gives us
\begin{equation}
    c 
    = 
    \frac{a^2 (1-\bar{\alpha}_t)}{b^2}
    =
    \frac{\bar{\alpha}_s (1-\bar{\alpha}_t)}
         {\bar{\alpha}_t (1-\bar{\alpha}_s)} 
    \label{eq:c_backtracking}
\end{equation}

We first obtain the value of $\bar{\alpha}_s$ that satisfies Equation~\ref{eq:c_backtracking} for a given choice of $c$, and then obtain the corresponding timestep $s$. We sweep through values of $c$ in range $0 \leq c \leq 0.25$ (as mentioned in~\cite{jeong2025upsample}) to produce different values of $s$. We compute all such candidate values of $s$ and pick the best $s$ empirically. For each value of $c$, we generate the corrected samples $\tilde{x}_t^{R}$ and the corresponding DDPM forward samples $x_s^R$ using 2048 training images. We then compute the Jensen–Shannon divergence between these distributions to obtain the final backtracking index 
$s$ as the one that produces the minimum JS divergence.

This experiment serves as our validation of our proposed method \ourmodel{}. During training, we follow a specific resolution schedule, so that for each timestep $t$, the model receives a state $x_t^{r(t)}$. To support distribution correction, we additionally include timestep $s$, $\tilde{x}_t^{R}$ to the training samples. During inference, the denoising process follows the standard reverse diffusion trajectory, with the following change: whenever the process reaches a timestep that has an upsampling step, the model rolls back to a slightly earlier timestep and continues denoising from that point at the higher resolution. This experiment illustrates the computational advantages of operating at multiple resolutions, using an architecture like \ourmodel{}, as a lot of the early denoising occurs at lower spatial resolutions. However, this setup requires rollback around each upsampling point, creating overlapping steps in the reverse process. While this model provides computational savings, there is an additional overhead of denoising for additional timesteps. 

In Table~\ref{tab:flexi_unet_backtracking}, we show the FID values obtained for this experiment after training the model for 500k iterations. We compare the performance of ~\ourmodel{} trained with \oursshort{} to \ourmodel{} trained without \oursshort{}. We observe that \ourmodel{} with \oursshort{} has better FID values, while also being faster at inference. 

\begin{table}
\caption{Results of \ourmodel{}~(w/o \oursshort{}) on CelebA-64. Computed at 500k iterations. Inference time is computed as the average time (in minutes) to generate a batch of samples (256 samples).}
\centering
\resizebox{0.8\linewidth}{!}{
\label{tab:flexi_unet_backtracking}
\begin{tabular}{@{}lcc}
\toprule
Method & FID & Inference Time \\
 \midrule
\ourmodel{}~(w/o \oursshort{}, Equal, 2L) & 2.44 & 15.52\\
\ourmodel{}~(w/o \oursshort{}, Equal, 4L) & 5.79 & 13.32 \\
\ourmodel{}~(w/ \oursshort{}, ConvexDecay0.5, 2L) & 2.26 & 14.98 \\
\ourmodel{}~(w/ \oursshort{}, ConvexDecay0.5, 4L) &  4.87 & 11.20\\
\bottomrule
\end{tabular}
}
\end{table}

\subsection{Resolution Schedules}
\label{sec:res_sch}
Here, we will define the functions that we used for the resolution schedules. We define what the resolution of the image should be given the diffusion timestep $t$, using a function $r(t)$. As shown in Fig.~\ref{fig:r-t_schedule}, we use a discrete version of the resolution schedule, but it is based on a continuous function. Suppose for the discrete version we use a list of resolutions $[r_{\text{min}},2r_{\text{min}},\ldots, 2^{n-2}r_{\text{min}}, 2^{n-1}r_{\text{min}}] $ where $r_{\text{min}}$ is the smallest resolution and $n$ is the number of resolutions. For the continuous version, let's first define normalized time $\tau=t/(T-1)$, where $T$ denotes the number of diffusion states. Then the normalized time to resolution schedule is defined as:
$$r_{\text{cont}}(\tau)=r_{\text{min}}\cdot2^{(n-1)f(\tau)}$$
where $f(\tau)$ is the exponential schedule function that works as the multiplier to the exponent of $2$. For example, when $f(\tau)=0$, then $r_{\text{cont}}(\tau)=r_{\text{min}}$, while when $f(\tau)=1$, then $r_{\text{cont}}(\tau)=r_{\text{max}}=2^{n-1}r_{\text{min}}$.

For the discrete version, we want to similarly sample from $R=[r_{\text{min}},2r_{\text{min}},\ldots, 2^{n-2}r_{\text{min}}, 2^{n-1}r_{\text{min}}=r_{\text{max}}]$, using the same schedule but over these discrete values. So, here we instead index the schedule function $i(\tau)$ that gives the index to select from $R$ given $\tau$. 
$$r(\tau)=R[i(\tau)]$$
Similar to $f$, when $i(\tau)=0$, we have $r(\tau)=r_{\text{min}}$, and when $i(\tau)=1$, $r(\tau) = r_{\text{max}}$. Now we can introduce our schedules.    
\subsubsection{Equal}
This is the easiest linear schedule.
\begin{itemize}
    \item Continuous: $f(\tau)=1-\tau$
    \item Discrete: $i(\tau)=n-1-\lfloor n\tau \rfloor$
\end{itemize}
\subsubsection{ConvexDecay\_$\gamma$}
With a $\gamma>0$ parameter, this function can simulate a convex or concave function depending on this parameter. 
\begin{itemize}
    \item Continuous: $f(\tau)=1-(1-\tau)^\gamma$
    \item Discrete: $i(\tau)=n-1-\lfloor nf(\tau) \rfloor$
\end{itemize}
For $\gamma>1$, it shows slow decay first, then faster, while for $\gamma<1$, fast decay first, then slower.
\subsubsection{TanhLikeDecay\_$\gamma$}
Here we wanted a function that looks like $\text{tanh}(\cdot)$ function, which is steep at the highest and the lowest timesteps but is flat in the middle. This essentially spends more time in the middle resolutions. We approximate this using a polynomial. 

First, we define a polynomial over a variable $u\in[-0.5,0.5]$ as follows:
$$x(u,\gamma)=\text{sign}(u)|u|^\gamma+0.5$$
$$p(x)=-2x^3+3x^2-0.5$$
The polynomial $p(x)$ is monotonically increasing in the range of $[-0.5, 0.5]$ for $x\in[0,1]$, while $x(u)$ is a function that looks like the $\text{tanh}(\cdot)$ function but is centered around $0.5$. Essentially, $p(x(u))$ looks like the $\text{tanh}$ shape and is centered around the origin, but has varying range dependent on $\gamma$. We want this function to be equal to 1 at $u=0.5$ and -1 at $u=-0.5$. To achieve that, we normalize this function:
$$\hat{p}(u, \gamma)=\frac{p(x(u, \gamma))}{p(x(0.5, \gamma))}$$
Finally, to shift this function from $[-0.5,0.5]\rightarrow[-1,1]$ to $[0,1]\rightarrow[0,1]$, we apply the following transformation:
$$\text{tanh\_like}(u, \gamma)=0.5\cdot\hat{p}(x(u-0.5, \gamma))+0.5$$
\begin{figure}[h!]
    \centering
    \includegraphics[width=\linewidth]{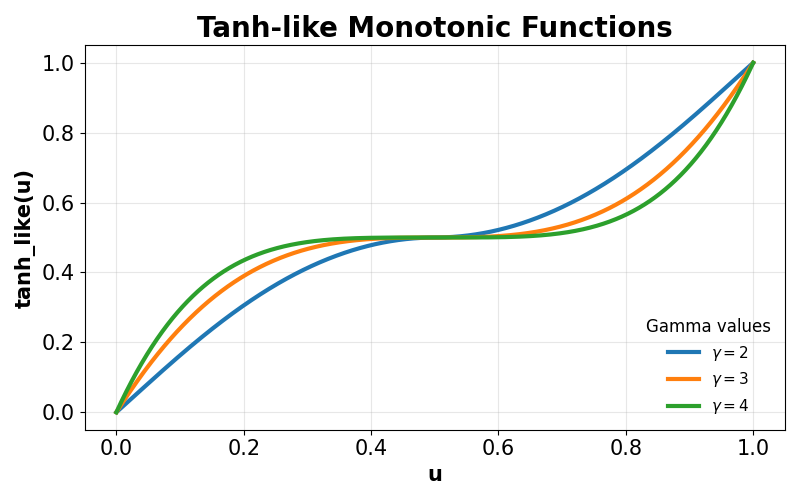}
    \caption{Visualization of $\text{tanh\_like}(\cdot)$ for different $\gamma$'s.}
    \label{fig:tanh_like_styles}
\end{figure}

Now, based on this definition, we can define the schedule.
\begin{itemize}
    \item Continuous: $f(\tau)=1-\text{tanh\_like}(1-\tau, \gamma)$
    \item Discrete: $i(\tau)=\lfloor nf(\tau) \rfloor$
\end{itemize}
\subsubsection{SigmoidLikeDecay\_$\gamma$}
Here we want a simoid-like curve, \ie, steep in the middle while flatter at the beginning and the end.
We can such a curve by inverting $\hat{p}(\cdot)$. Following similar stretching and normalization, we can define another function that goes from $[-0.5,0.5]\rightarrow[-1,1]$ as:
$$\hat{h}(u) = \frac{(0.5\cdot\hat{p})^{-1}(u, \gamma)}{(0.5\cdot\hat{p})^{-1}(-0.5, \gamma)}$$.
Using the same shifting to transform $[-0.5,0.5]\rightarrow[-1,1]$ to $[0,1]\rightarrow[0,1]$, we have:
$$\text{sigmoid\_like}(u, \gamma)=0.5\cdot\hat{h}(x(u-0.5, \gamma))+0.5$$
Finally, we define the schedule.
\begin{itemize}
    \item Continuous: $f(\tau)=1-\text{sigmoid\_like}(1-\tau, \gamma)$
    \item Discrete: $i(\tau)=\lfloor nf(\tau) \rfloor$
\end{itemize}

\subsection{More Comparisons}
\label{sec:more_comparisons}
\noindent \textbf{Upsampling Diffusion Probabilistic Models (UDPM)~\cite{abu2023udpm}.} The mathematical formulation and implementation of \oursshort~ can be viewed as a generalization of UDPM. However, UDPM should be considered as a GAN instead of diffusion, as their performance degrades without perceptual and adversarial losses (Table~\ref{tab:udpm}) and generations are washed out (Fig.~\ref{fig:udpm}, right). Nonetheless, even without extra losses, \oursshort~outperforms UDPM in FID and training time. 
Furthermore, UDPM has not been tested at resolutions higher than 64.

\noindent \textbf{Latent Diffusion Models (LDM)~\cite{rombach2022ldm}.} LDMs operate in latent space with different architectures and rely on a compute-intensive pipeline, including two-stage VAE training on large-scale datasets such as OpenImages, making fair comparison difficult. Nonetheless, Table~\ref{tab:LDM} shows \oursshort~ (6L) is faster than LDM. Moreover, \oursshort~ can be applied in latent space as a multi-resolution interpolation degradation, enabling more efficient Scale Space LDMs.

\noindent \textbf{Low-res diffusion + super-res.} Another baseline could be using a low-resolution generation and applying a super-resolution model over it. Table~\ref{tab:superres} shows that even with a pretrained LDM super-res model trained for 3$\times$ more iterations, and on a large dataset, \oursshort~has better performance. Adding multiple stages normally leads to distribution shifts as well as more inference steps coming from different stages. 

\noindent \textbf{PixelFlow~\cite{chen2025pixelflow}, DFM~\cite{haji2025decomposable}.} These are flow-based DiT models with differential equation solver-based sampling, and hence, are hard to compare fairly against. Nonetheless, in a fair setting in section~\ref{sec:our_arch_validity}, we recreate a multi-res pixel diffusion similar to PixelFlow, and show that using ~\ourmodel{} formulation outperforms it (Table~\ref{tab:flexi_unet_backtracking}).

\begin{table}[t]
    \begin{minipage}{.56\linewidth}
    \caption{FID comparison of SSD and super-resolved DDPM (64 res.) using an OpenImages pretrained 4$\times$ LDM.  }
      \vspace{-1em}
\label{tab:superres}
      \centering
      \resizebox{0.85\textwidth}{!}{
      \setlength\tabcolsep{4pt}
      \begin{tabular}{@{}lc@{}} 
        \toprule
        Method & FID\\
        \midrule
        SSD (3L, res. 256) & 7.79 \\
        low-res diffusion (res. 64) + super-res (4$\times$) & 7.91\\
        \bottomrule
    \end{tabular}
    }
    \end{minipage} 
    \hfill
    \begin{minipage}{.01\linewidth}
    \end{minipage}
    \begin{minipage}{.41\linewidth}
    \caption{Inference time per batch: SSD and LDMs (1000 steps, bs=32, A4000).}
      \vspace{-1em}
\label{tab:LDM}
      \centering
      \resizebox{\textwidth}{!}{
      \setlength\tabcolsep{4pt}
      \begin{tabular}{@{}lc@{}} 
        \toprule
        Method & Inference time (secs)\\
        \midrule
        SSD (6L, res. 256) & 495 \\
        LDM  (res. 256) & 515\\
        \bottomrule
    \end{tabular}
    }
    \end{minipage}%
\end{table}

\begin{table}[t]
    \begin{minipage}{\linewidth}
      \caption{Comparison with UDPM at 64 resolution: FID, training (1 H100), and inference speed (1 A4000, bs=256)}
      \vspace{-0.5em}
\label{tab:udpm}
      \centering
      \resizebox{0.9\textwidth}{!}{
      \setlength\tabcolsep{4pt}
        \begin{tabular}{@{}lcccc@{}} 
        \toprule
         Method & Inference  & Inference  & Training Time  & FID   \\
         & steps & Time / batch & (250K iters) & (250K iters)\\
         &  & (in secs) & (in hours) & \\
        \midrule
        DDPM-$\epsilon$ & 1000 & 1018.07 & 17.575 & 2.36 \\
        SSD (2L) & 1000 & 898.71 & 15.658 & 2.68 \\
        SSD (4L) & 1000 & 672.09 & 13.095 & 4.1 \\
        \midrule
        UDPM & 3 &  1.88  & 30.58& 7.51 \\
        UDMP (w/o Adv. \& Perceptual loss) & 3 & 1.84 & 31.63 & 98.61 \\
        
        \bottomrule
    \end{tabular}
    }
    \end{minipage}
\end{table} 
\begin{figure}[!t]
    \centering
    \vspace{-0.3em}
\includegraphics[width=0.9\linewidth]{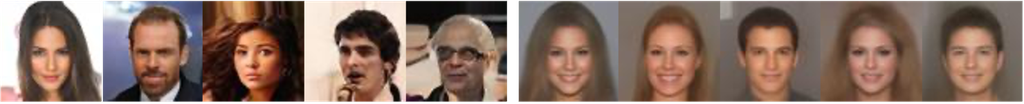}
    \vspace{-.5em}
    \caption{\footnotesize  UDPM generations w/ (left) and w/o (right) adversarial and perceptual losses.}
    \label{fig:udpm}
    \vspace{-2em}
\end{figure}

\subsection{Quantitative Results}
\label{sec:quant_results}

\noindent \textbf{Number of Inference Steps.} In Table \ref{tab:infer_time}, we compare inference speed across different samplers and denoising steps. We report DDPM sampling with the default 1000 steps, a reduced 250 step process, and DDIM with 25 steps. For our method, we report results with 1000 and 250 DDPM steps, since \oursshort{} is formulated in the DDPM setting. We observe that reducing the number of diffusion steps leads to a much larger performance degradation for DDPM-$\epsilon$ and DDPM-$x_0$ than for our approach. This aligns with prior observations that DDPM-$\epsilon$ models trained with the $L_{\text{simple}}$ loss (with fixed sigmas) deteriorate substantially when the number of sampling steps is reduced~\cite{nichol2021improved}, which is reflected in our results as well.

We note that \oursshort{} degrades far less when reducing the sampling steps to 250, while also providing substantial inference speedups. However, to ensure a fair comparison against baselines, we report all final quantitative results in the paper using the standard 1000-step setting. The speedup column in Table~\ref{tab:infer_time} reports the speedup obtained in generating a batch of 256 samples relative to the time taken by DDPM-$x_0$.

\noindent \textbf{Lanczos sampling overhead.} Table~\ref{tab:lanczos} shows that the overhead of using Lanczos instead of torch.randn call is negligible, since it is applied only in the resolution-changing steps (1${\times}$ in \oursshort~(2L),  2${\times}$ in \oursshort~(3L)). Refer to Table~\ref{tab:infer_time} for comparison against DDPM. 

\subsection{Qualitative Results}

We show qualitative results of \oursshort~on every setting that we have trained and noted in Tables~\ref{tab:celeba_main_table}, and \ref{tab:imagenet64_main_table}. For every setting, we show the progression of noisy states $x_t$ and predicted clean images $x_{0,\theta}^{r(t-1)}$ during generation, and a grid of generated images. The results start on the next page.

\begin{figure*}[h!]
    \centering
    \includegraphics[width=\linewidth]{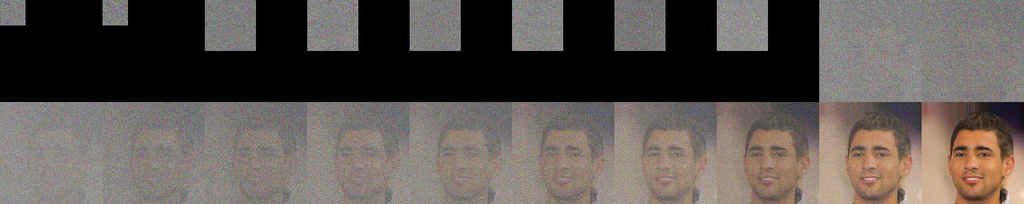}
    \caption{Progression of noisy states $x_t$ during generation using \oursshort~(\ourmodel{}, 3L) on CelebA-256.}
    \label{fig:progression_celeba_256_3L_xt}
\end{figure*}
\begin{figure*}[h!]
    \centering
    \includegraphics[width=\linewidth]{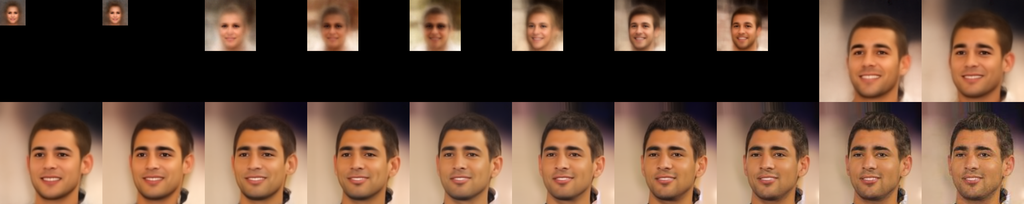}
    \caption{Progression of predicted clean images $x_{0,\theta}^{r(t-1)}$ during generation using \oursshort~(\ourmodel{}, 3L) on CelebA-256.}
    \label{fig:progression_celeba_256_3L_x0}
\end{figure*}
\begin{figure*}[h!]
    \centering
    \includegraphics[width=\linewidth]{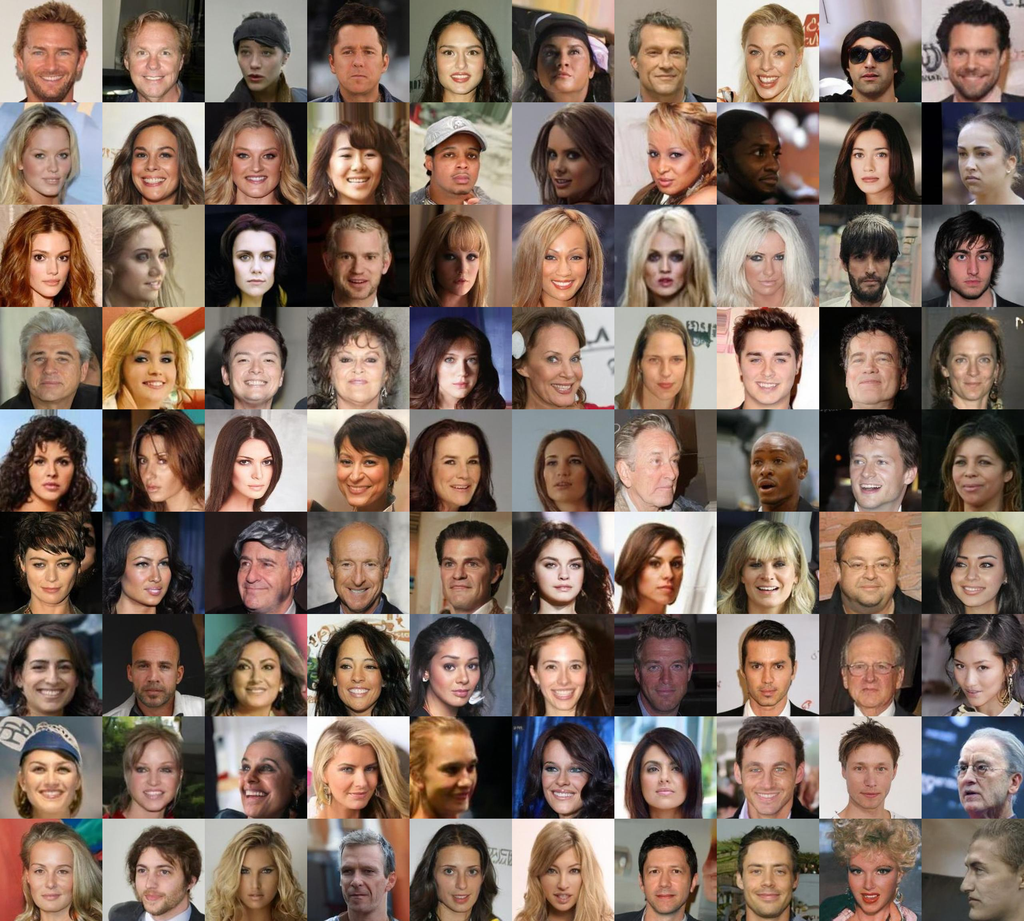}
    \caption{Generated Samples using \oursshort~(\ourmodel{}, 3L) on CelebA-256.}
    \label{fig:gen_grid_celeba_256_3L}
\end{figure*}

\begin{figure*}[h!]
    \centering
    \includegraphics[width=\linewidth]{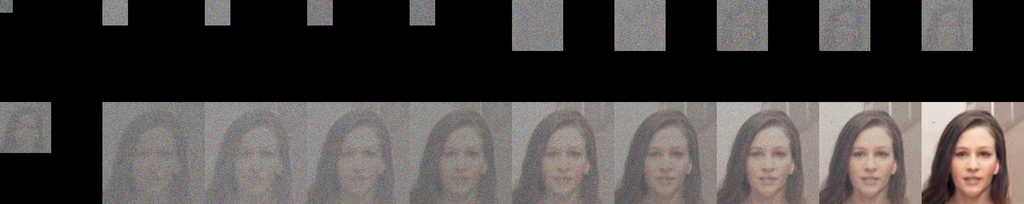}
    \caption{Progression of noisy states $x_t$ during generation using \oursshort~(\ourmodel{}, 4L) on CelebA-256.}
    \label{fig:progression_celeba_256_4L_xt}
\end{figure*}
\begin{figure*}[h!]
    \centering
    \includegraphics[width=\linewidth]{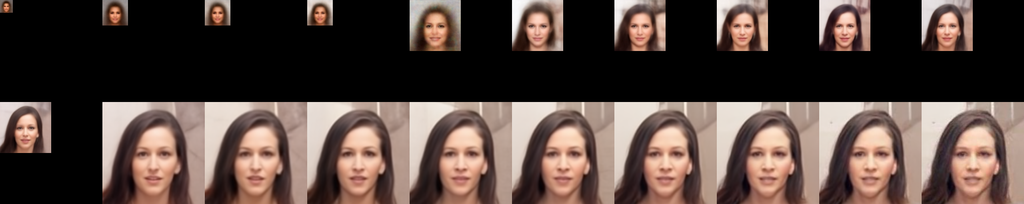}
    \caption{Progression of predicted clean images $x_{0,\theta}^{r(t-1)}$ during generation using \oursshort~(\ourmodel{}, 4L) on CelebA-256.}
    \label{fig:progression_celeba_256_4L_x0}
\end{figure*}
\begin{figure*}[h!]
    \centering
    \includegraphics[width=\linewidth]{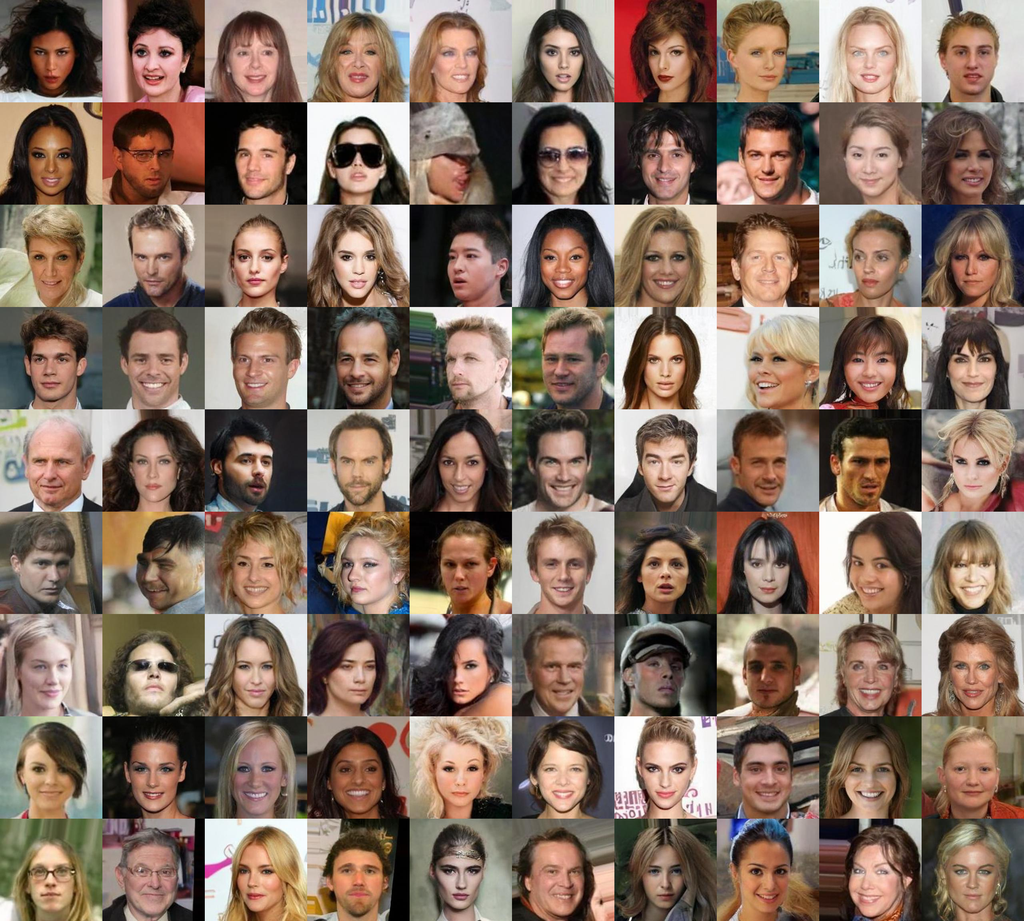}
    \caption{Generated Samples using \oursshort~(\ourmodel{}, 4L) on CelebA-256.}
    \label{fig:gen_grid_celeba_256_4L}
\end{figure*}

\begin{figure*}[h!]
    \centering
    \includegraphics[width=\linewidth]{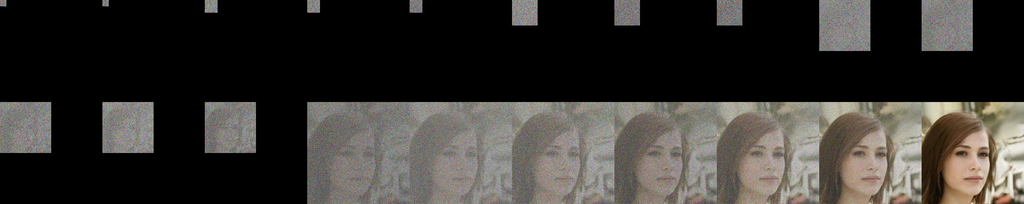}
    \caption{Progression of noisy states $x_t$ during generation using \oursshort~(\ourmodel{}, 6L) on CelebA-256.}
    \label{fig:progression_celeba_256_6L_xt}
\end{figure*}
\begin{figure*}[h!]
    \centering
    \includegraphics[width=\linewidth]{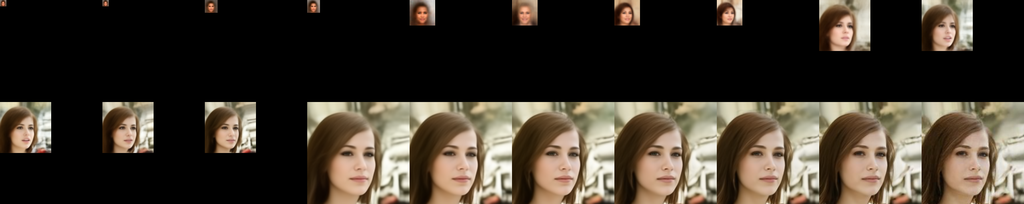}
    \caption{Progression of predicted clean images $x_{0,\theta}^{r(t-1)}$ during generation using \oursshort~(\ourmodel{}, 6L) on CelebA-256.}
    \label{fig:progression_celeba_256_6L_x0}
\end{figure*}
\begin{figure*}[h!]
    \centering
    \includegraphics[width=\linewidth]{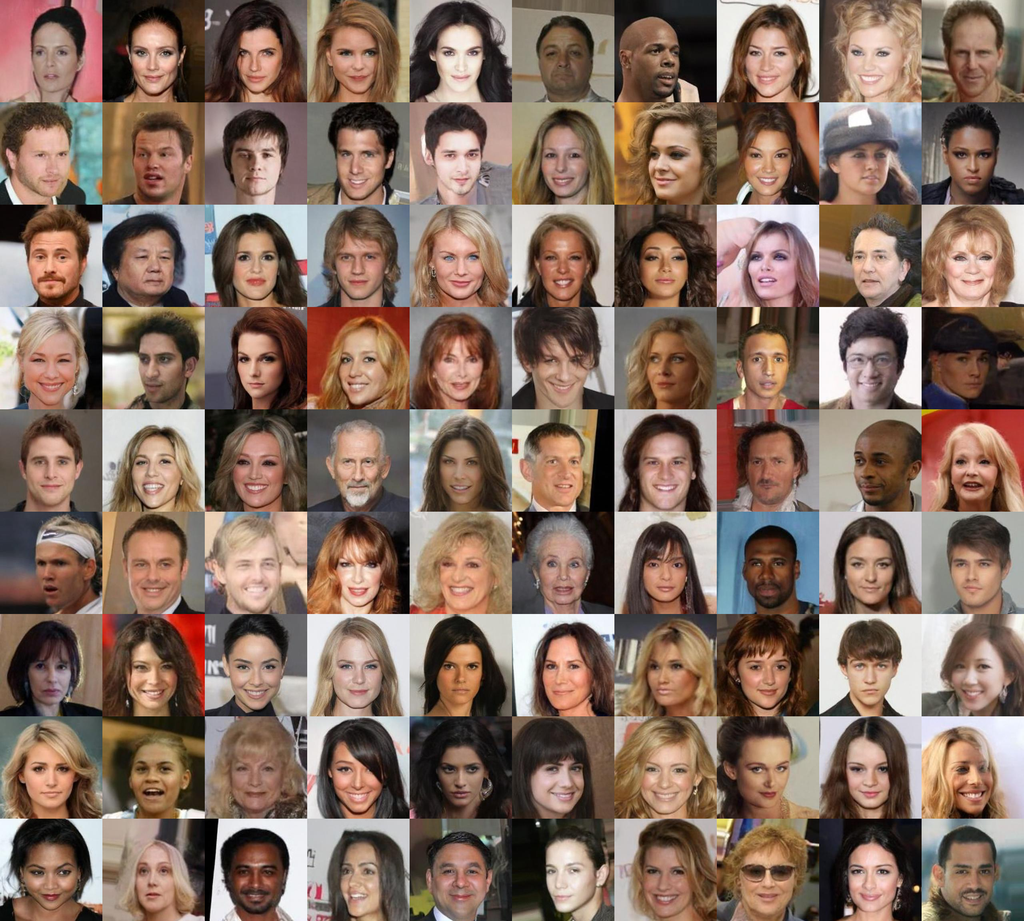}
    \caption{Generated Samples using \oursshort~(\ourmodel{}, 6L) on CelebA-256.}
    \label{fig:gen_grid_celeba_256_6L}
\end{figure*}

\begin{figure*}[h!]
    \centering
    \includegraphics[width=\linewidth]{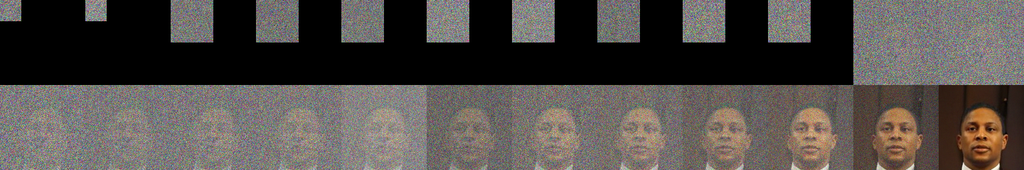}
    \caption{Progression of noisy states $x_t$ during generation using \oursshort~(\ourmodel{}, 3L) on CelebA-128.}
    \label{fig:progression_celeba_128_3L_xt}
\end{figure*}
\begin{figure*}[h!]
    \centering
    \includegraphics[width=\linewidth]{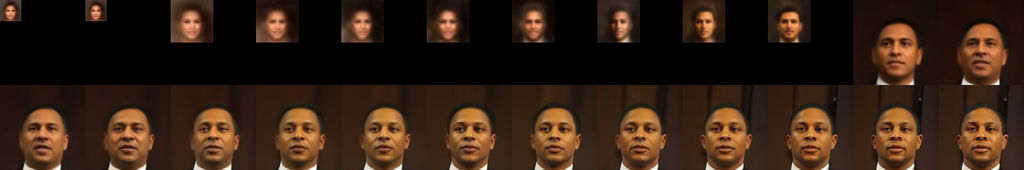}
    \caption{Progression of predicted clean images $x_{0,\theta}^{r(t-1)}$ during generation using \oursshort~(\ourmodel{}, 3L) on CelebA-128.}
    \label{fig:progression_celeba_128_3L_x0}
\end{figure*}
\begin{figure*}[h!]
    \centering
    \includegraphics[width=\linewidth]{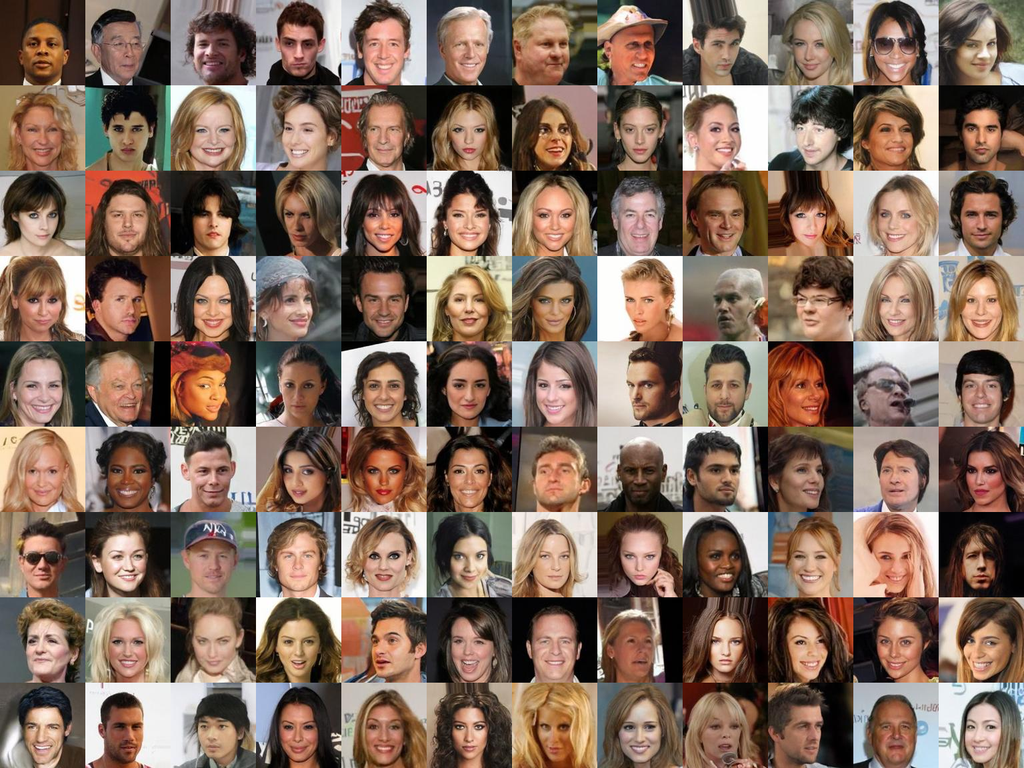}
    \caption{Generated Samples using \oursshort~(\ourmodel{}, 3L) on CelebA-128.}
    \label{fig:gen_grid_celeba_128_3L}
\end{figure*}

\begin{figure*}[h!]
    \centering
    \includegraphics[width=\linewidth]{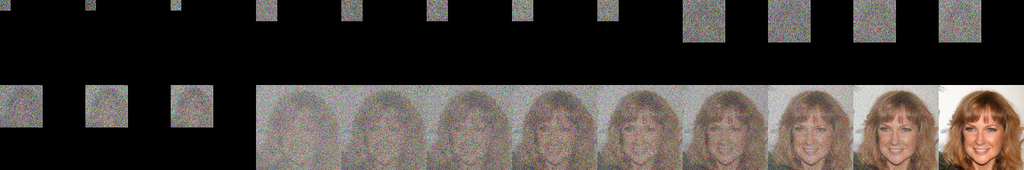}
    \caption{Progression of noisy states $x_t$ during generation using \oursshort~(\ourmodel{}, 5L) on CelebA-128.}
    \label{fig:progression_celeba_128_5L_xt}
\end{figure*}
\begin{figure*}[h!]
    \centering
    \includegraphics[width=\linewidth]{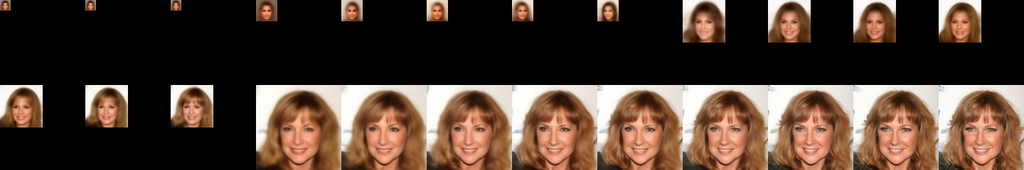}
    \caption{Progression of predicted clean images $x_{0,\theta}^{r(t-1)}$ during generation using \oursshort~(\ourmodel{}, 5L) on CelebA-128.}
    \label{fig:progression_celeba_128_5L_x0}
\end{figure*}
\begin{figure*}[h!]
    \centering
    \includegraphics[width=\linewidth]{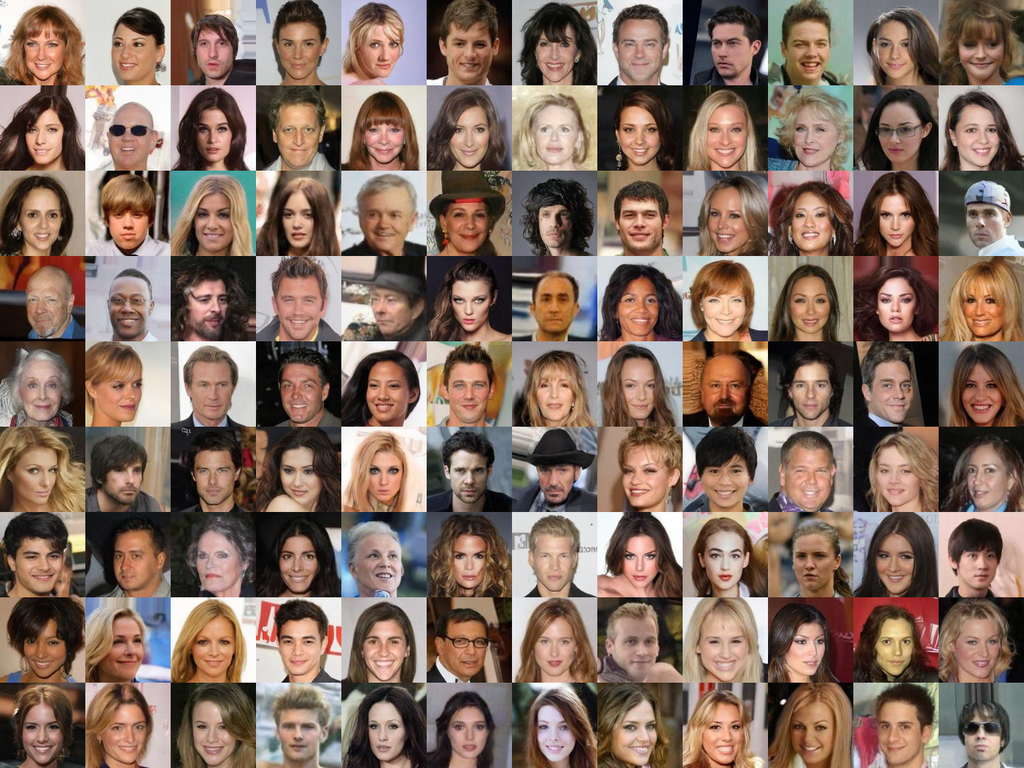}
    \caption{Generated Samples using \oursshort~(\ourmodel{}, 5L) on CelebA-128.}
    \label{fig:gen_grid_celeba_128_5L}
\end{figure*}

\begin{figure*}[h!]
    \centering
    \includegraphics[width=\linewidth]{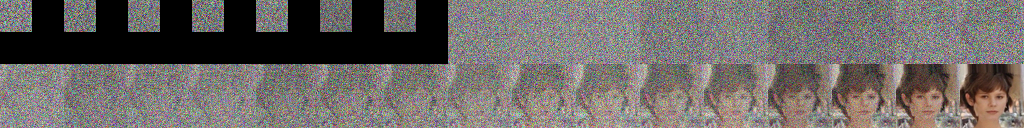}
    \caption{Progression of noisy states $x_t$ during generation using \oursshort~(\ourmodel{}, 2L) on CelebA-64.}
    \label{fig:progression_celeba_64_2L_xt}
\end{figure*}
\begin{figure*}[h!]
    \centering
    \includegraphics[width=\linewidth]{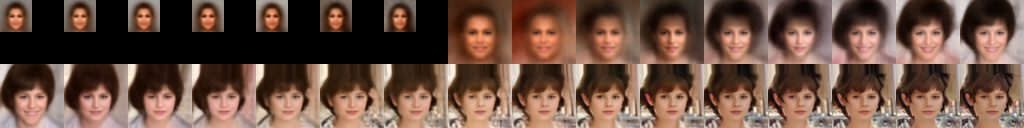}
    \caption{Progression of predicted clean images $x_{0,\theta}^{r(t-1)}$ during generation using \oursshort~(\ourmodel{}, 2L) on CelebA-64.}
    \label{fig:progression_celeba_64_2L_x0}
\end{figure*}
\begin{figure*}[h!]
    \centering
    \includegraphics[width=\linewidth]{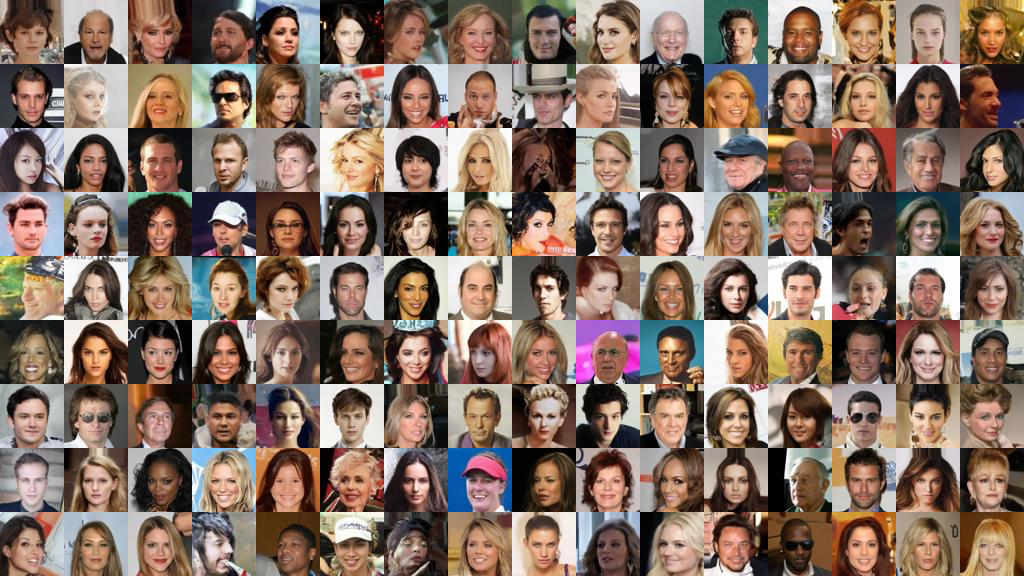}
    \caption{Generated Samples using \oursshort~(\ourmodel{}, 2L) on CelebA-64.}
    \label{fig:gen_grid_celeba_64_2L}
\end{figure*}

\begin{figure*}[h!]
    \centering
    \includegraphics[width=\linewidth]{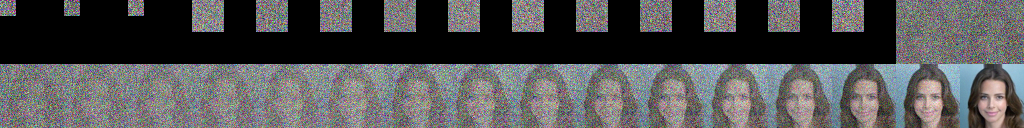}
    \caption{Progression of noisy states $x_t$ during generation using \oursshort~(\ourmodel{}, 3L) on CelebA-64.}
    \label{fig:progression_celeba_64_3L_xt}
\end{figure*}
\begin{figure*}[h!]
    \centering
    \includegraphics[width=\linewidth]{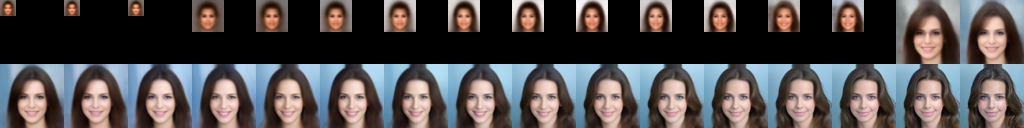}
    \caption{Progression of predicted clean images $x_{0,\theta}^{r(t-1)}$ during generation using \oursshort~(\ourmodel{}, 3L) on CelebA-64.}
    \label{fig:progression_celeba_64_3L_x0}
\end{figure*}
\begin{figure*}[h!]
    \centering
    \includegraphics[width=\linewidth]{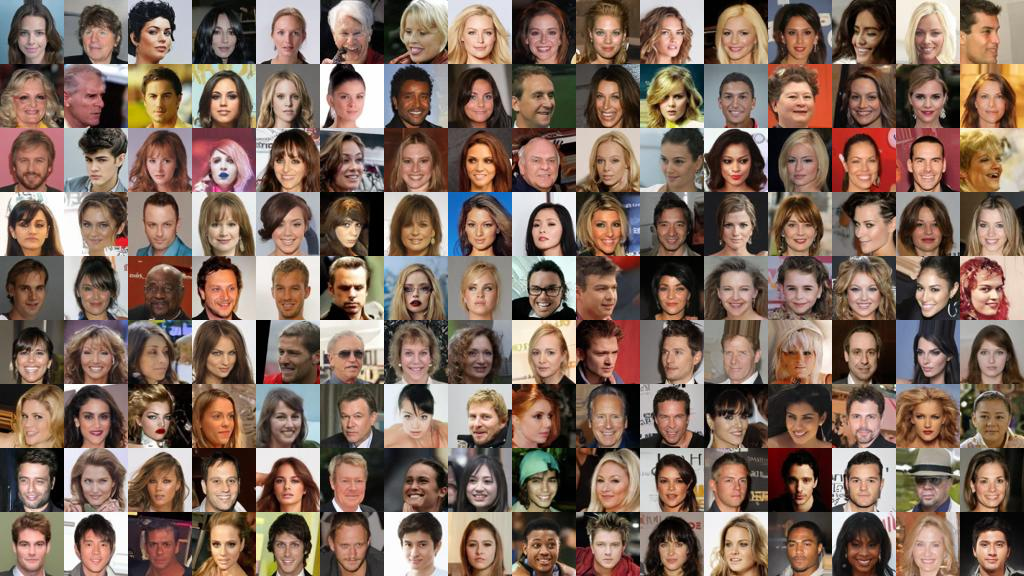}
    \caption{Generated Samples using \oursshort~(\ourmodel{}, 3L) on CelebA-64.}
    \label{fig:gen_grid_celeba_64_3L}
\end{figure*}

\begin{figure*}[h!]
    \centering
    \includegraphics[width=\linewidth]{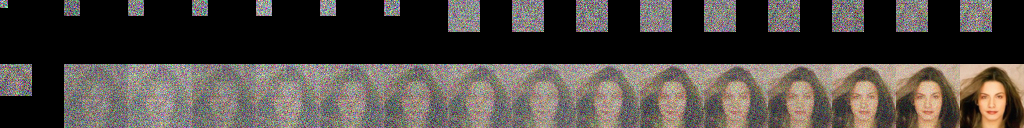}
    \caption{Progression of noisy states $x_t$ during generation using \oursshort~(\ourmodel{}, 4L) on CelebA-64.}
    \label{fig:progression_celeba_64_4L_xt}
\end{figure*}
\begin{figure*}[h!]
    \centering
    \includegraphics[width=\linewidth]{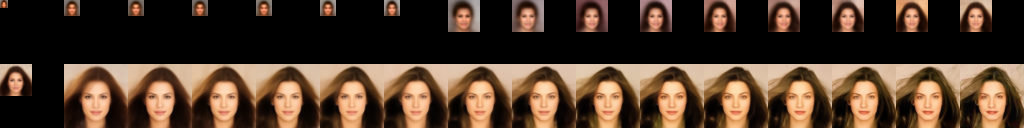}
    \caption{Progression of predicted clean images $x_{0,\theta}^{r(t-1)}$ during generation using \oursshort~(\ourmodel{}, 4L) on CelebA-64.}
    \label{fig:progression_celeba_64_4L_x0}
\end{figure*}
\begin{figure*}[h!]
    \centering
    \includegraphics[width=\linewidth]{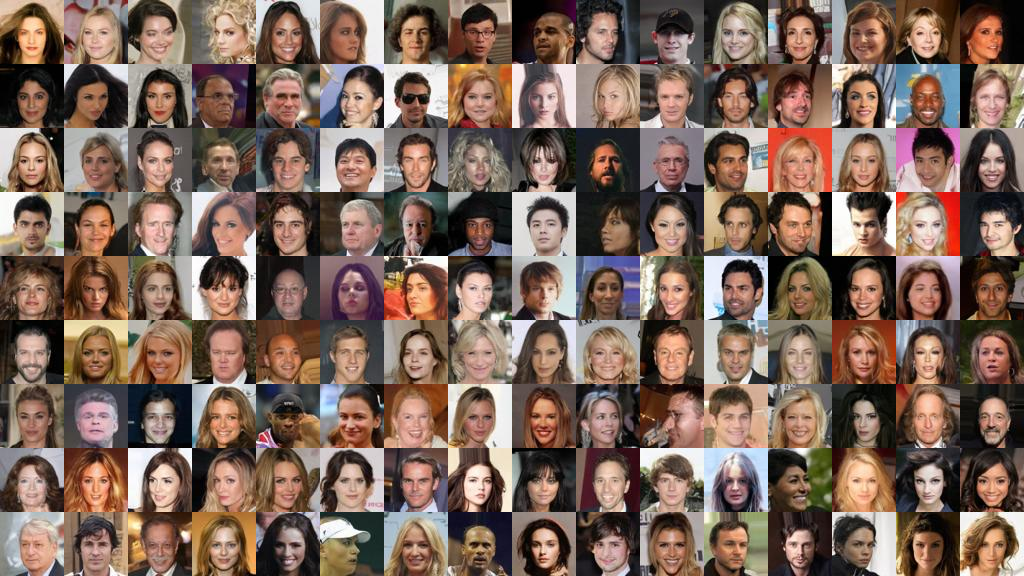}
    \caption{Generated Samples using \oursshort~(\ourmodel{}, 4L) on CelebA-64.}
    \label{fig:gen_grid_celeba_64_4L}
\end{figure*}

\begin{figure*}[h!]
    \centering
    \includegraphics[width=\linewidth]{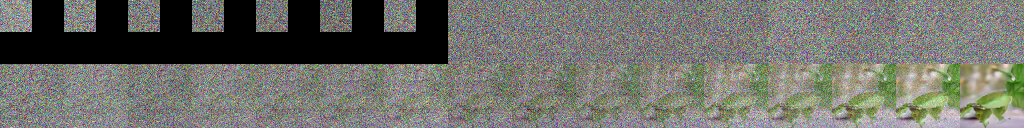}
    \includegraphics[width=\linewidth]{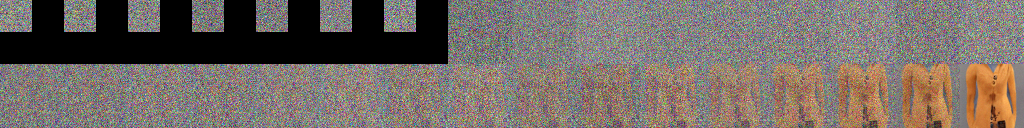}
    \includegraphics[width=\linewidth]{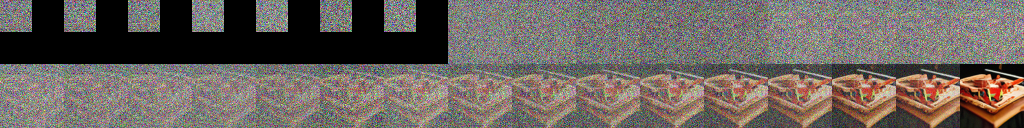}
    \caption{Progression of noisy states $x_t$ during generation using \oursshort~(\ourmodel{}, 2L) on ImageNet-64. Here we show the progression of 3 samples; each pair of rows corresponds to a single sample.}
    \label{fig:progression_imagenet_64_2L_xt}
\end{figure*}
\begin{figure*}[h!]
    \centering
    \includegraphics[width=\linewidth]{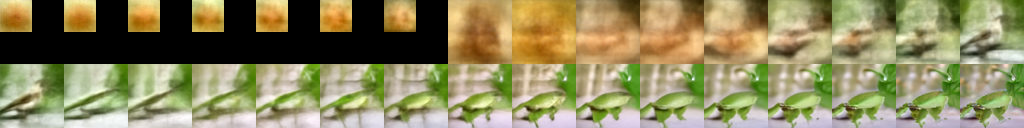}
    \includegraphics[width=\linewidth]{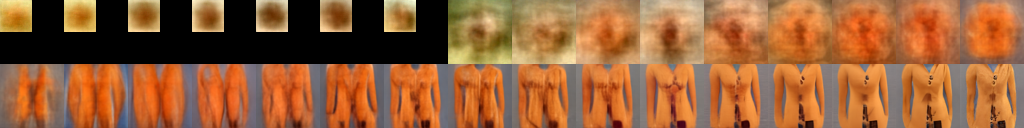}
    \includegraphics[width=\linewidth]{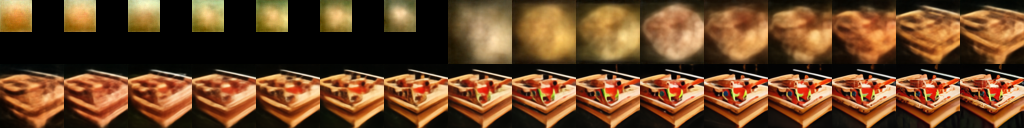}
    \caption{Progression of predicted clean images $x_{0,\theta}^{r(t-1)}$ during generation using \oursshort~(\ourmodel{}, 2L) on ImageNet-64. Here we show the progression of 3 samples; each pair of rows corresponds to a single sample.}
    \label{fig:progression_imagenet_64_2L_x0}
\end{figure*}

\begin{figure*}[h!]
    \centering
    \includegraphics[width=\linewidth]{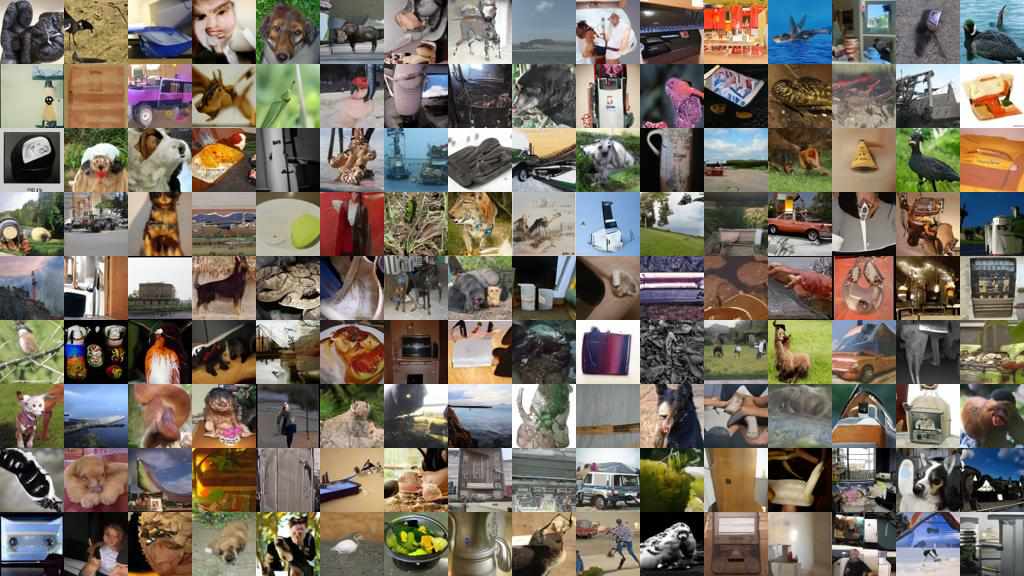}
    \caption{Generated Samples using \oursshort~(\ourmodel{}, 2L) on ImageNet-64.}
    \label{fig:gen_grid_imagenet_64_2L}
\end{figure*}

\begin{figure*}[h!]
    \centering
    \includegraphics[width=\linewidth]{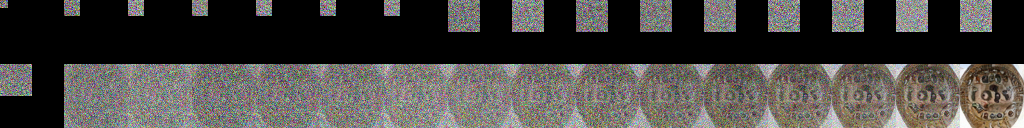}
    \includegraphics[width=\linewidth]{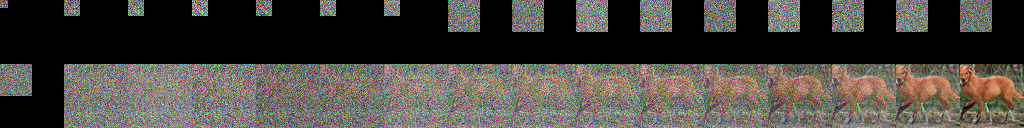}
    \includegraphics[width=\linewidth]{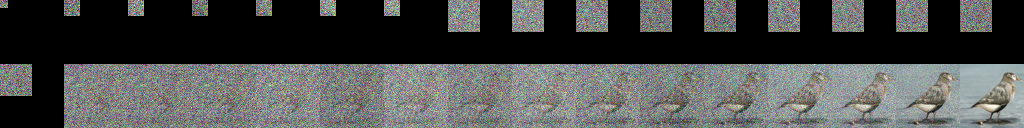}
    \caption{Progression of noisy states $x_t$ during generation using \oursshort~(\ourmodel{}, 4L) on ImageNet-64. Here we show the progression of 3 samples; each pair of rows corresponds to a single sample.}
    \label{fig:progression_imagenet_64_4L_xt}
\end{figure*}
\begin{figure*}[h!]
    \centering
    \includegraphics[width=\linewidth]{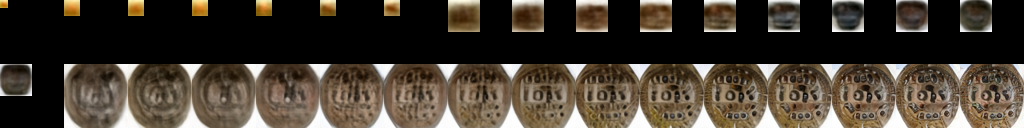}
    \includegraphics[width=\linewidth]{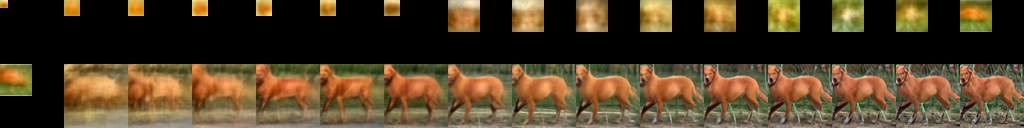}
    \includegraphics[width=\linewidth]{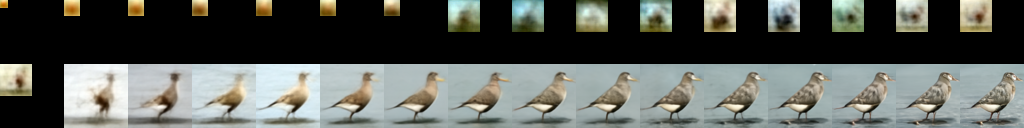}
    \caption{Progression of predicted clean images $x_{0,\theta}^{r(t-1)}$ during generation using \oursshort~(\ourmodel{}, 4L) on ImageNet-64. Here we show the progression of 3 samples; each pair of rows corresponds to a single sample.}
    \label{fig:progression_imagenet_64_4L_x0}
\end{figure*}

\begin{figure*}[h!]
    \centering
    \includegraphics[width=\linewidth]{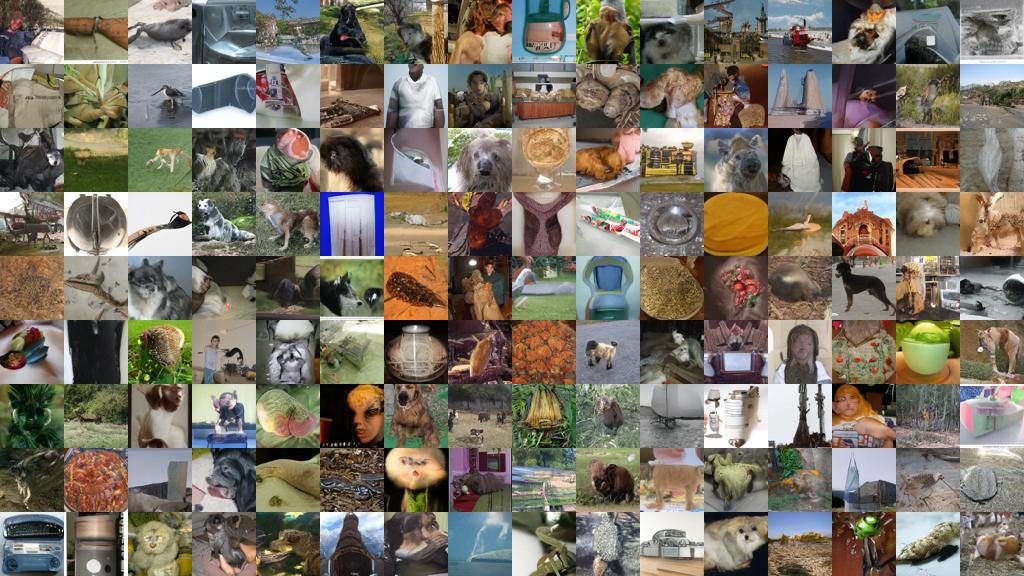}
    \caption{Generated Samples using \oursshort~(\ourmodel{}, 4L) on ImageNet-64.}
    \label{fig:gen_grid_imagenet_64_4L}
\end{figure*}

\onecolumn

\section{Mathematical Derivations}

In this section, we provide derivations for various mathematical results provided in the main paper.
\subsection{Forward Transition}
\begin{theorem}[Forward Transition]
\label{thm:forward_covariance_linear_diffusion}
Let a generalized linear diffusion process be defined by
\begin{align}
x_t = M_t x_{t-1} + \eta_t,
\qquad
\eta_t \sim \mathcal{N}(0,\Sigma_{t|t-1}),
\end{align}
and suppose the marginal distribution satisfies
\begin{align}
q(x_t \mid x_0) = \mathcal{N}(\mu_t, \Sigma_t).
\end{align}
Then the transition mean and covariance are given by
\begin{align}
\mu_t&=M_{1:t}x_0 \\
\Sigma_{t|t-1}
&=
\Sigma_t - M_t \Sigma_{t-1} M_t^T.
\end{align}
\end{theorem}

\begin{proof}
The mean part is true by design of the cumulative linear operator. 
Here we derive $\Sigma_{t|t-1}$.

\begin{align}
x_t &= M_t x_{t-1} + \eta_t, 
\quad \eta_t \sim \mathcal{N}(0, \Sigma_{t|t-1}) \nonumber\\
&= M_t (M_{1:t-1} x_0 + \epsilon_{t-1}) + \eta_t, 
\quad \epsilon_{t-1} \sim \mathcal{N}(0, \Sigma_{t-1}) \nonumber\\
&= M_{1:t} x_0 + M_t \epsilon_{t-1} + \eta_t. \nonumber
\end{align}

Hence,
\begin{align}
\mathrm{Cov}(x_t \mid x_0) 
&= \mathrm{Cov}(M_t \epsilon_{t-1} + \eta_t \mid x_0) \nonumber\\
&= \mathrm{Cov}(M_t \epsilon_{t-1} \mid x_0) 
  + \mathrm{Cov}(\eta_t) 
  \quad (\text{independence}) \nonumber\\
&= M_t \Sigma_{t-1} M_t^T + \Sigma_{t|t-1}. \nonumber \\
\Sigma_t &= M_t \Sigma_{t-1} M_t^T + \Sigma_{t|t-1}, \nonumber\\
\implies \Sigma_{t|t-1} &= \Sigma_t - M_t \Sigma_{t-1} M_t^T. \nonumber
\end{align}
\end{proof}

\subsection{Posterior Distribution}
\label{sec:derivation_reverse_step_ddpm_matrix}
\begin{theorem}[Posterior Distribution]
\label{thm:posterior_linear_diffusion}
Consider the linear generalized linear diffusion process
\begin{align}
x_t = M_t x_{t-1} + \eta_t,
\qquad
\eta_t \sim \mathcal{N}(0,\Sigma_{t|t-1}),
\end{align}
with marginals
\begin{align}
q(x_{t-1}\mid x_0)
&=
\mathcal{N}(\mu_{t-1},\Sigma_{t-1}), \\
q(x_t\mid x_0)
&=
\mathcal{N}(\mu_t,\Sigma_t).
\end{align}
Then the posterior distribution
\begin{align}
q(x_{t-1} \mid x_t, x_0)
\end{align}
is Gaussian:
\begin{align}
q(x_{t-1} \mid x_t, x_0)
=
\mathcal{N}(\mu_{t\rightarrow t-1}, \Sigma_{t\rightarrow t-1}),
\end{align}
with
\begin{align}
\Sigma_{t\rightarrow t-1}
&=
(\Sigma_{t-1}^{-1}
+
M_t^T\Sigma_{t|t-1}^{-1}M_t)^{-1}, \\
\mu_{t\rightarrow t-1}
&=
\Sigma_{t\rightarrow t-1}
\left(
\Sigma_{t-1}^{-1}\mu_{t-1}
+
M_t^T\Sigma_{t|t-1}^{-1}x_t
\right).
\end{align}
\end{theorem}

\begin{proof}
\begin{align}
q(x_{t-1}\!\mid x_t,x_0)
&= \frac{q(x_t\!\mid x_{t-1},x_0)\, q(x_{t-1}\!\mid x_0)}{q(x_t\!\mid x_0)} = \frac{q(x_t\!\mid x_{t-1})\, q(x_{t-1}\!\mid x_0)}{q(x_t\!\mid x_0)} \tag{a}\label{eq:markov} \\[4pt]
&\propto
\exp\!\Bigg(
    -\Big[
        (x_t - M_t x_{t-1})^\top \Sigma_{t|t-1}^{-1} (x_t - M_t x_{t-1})
    \Big]  -\Big[
        (x_{t-1}-\mu_{t-1})^\top \Sigma_{t-1}^{-1} (x_{t-1}-\mu_{t-1})
    \Big]  \tag{b}\label{eq:expand} \\
\nonumber
&\hspace{1.3cm}
    +\Big[
        (x_t-\mu_t)^\top \Sigma_{t}^{-1} (x_t-\mu_t)
    \Big]
\Bigg) \\[4pt]
\nonumber
&=
\exp\!\Bigg(
    -\Big[
        x_t^T\Sigma_{t|t-1}^{-1}x_t - (x_t^T\Sigma_{t|t-1}^{-1}M_t\begingroup\color{blue}x_{t-1}\endgroup + (x_t^T\Sigma_{t|t-1}^{-1}M_t\begingroup\color{blue}x_{t-1}\endgroup)^T) + \begingroup\color{red}x_{t-1}^T\endgroup M_t^T\Sigma_{t|t-1}^{-1}M_t\begingroup\color{red}x_{t-1}\endgroup
    \Big]  \\
    \nonumber
&\hspace{1.3cm}
    -\Big[
        \begingroup\color{red}x_{t-1}^T\endgroup\Sigma_{t-1}^{-1}\begingroup\color{red}x_{t-1}\endgroup-(\mu_{t-1}^T\Sigma_{t-1}^{-1}\begingroup\color{blue}x_{t-1}\endgroup+(\mu_{t-1}^T\Sigma_{t-1}^{-1}\begingroup\color{blue}x_{t-1}\endgroup)^T) + \mu_{t-1}^T\Sigma_{t-1}\mu_{t-1}
    \Big]  \\
    \nonumber
&\hspace{1.3cm}
    +C_1(x_0, x_t)
\Bigg) \\[4pt]
\nonumber
&=
\exp\!\Bigg(
    - \begingroup\color{red}x_{t-1}^T(\Sigma_{t-1}^{-1}+M_t^T\Sigma_{t|t-1}^{-1}M_t)x_{t-1}\endgroup
      \\
      \nonumber
&\hspace{1.3cm}
    +\Big[
        \begingroup\color{blue}(x_t^T\Sigma_{t|t-1}^{-1}M_t+\mu_{t-1}^T\Sigma_{t-1}^{-1})x_{t-1} + ((x_t^T\Sigma_{t|t-1}^{-1}M_t+\mu_{t-1}^T\Sigma_{t-1}^{-1})x_{t-1})^T \endgroup
    \Big]  \\
    \nonumber
&\hspace{1.3cm}
    +C_2(x_0, x_t)
\Bigg) \\[4pt]
\nonumber
\end{align}

In Eq.~\ref{eq:markov}, we first use Bayes' rule, and then use the Markov chain assumption. In Eq.~\ref{eq:expand}, we then substitute the marginal (Eq.~\ref{eq:make_noisy_ssd_matrix}) and forward transition (Eq.~\ref{eq:forward_step_ddpm_matrix}) distributions. Then we start collecting the terms quadratic (red) and linear (blue) in $x_{t-1}$.
From the quadratic and linear terms, we can complete the square and hence extract the mean and variance of the posterior normal distribution:
\begin{align}
    \Sigma_{t\rightarrow t-1} &=(\Sigma_{t-1}^{-1}+M_t^T\Sigma_{t|t-1}^{-1}M_t)^{-1} \nonumber\\
    \mu_{t\rightarrow t-1} &= \Sigma_{t\rightarrow t-1}(x_t^T\Sigma_{t|t-1}^{-1}M_t+\mu_{t-1}^T\Sigma_{t-1}^{-1})^T  \nonumber \\
    &=\Sigma_{t\rightarrow t-1}(M_t^T\Sigma_{t|t-1}^{-1}x_t+\Sigma_{t-1}^{-1}\mu_{t-1}) .  \nonumber \\
    \nonumber
\end{align}
The last step comes from the fact that for a symmetric matrix $A$, $(A^{-1})^T=A^{-1}$, and covariance matrices are symmetric.
\end{proof}

\subsection{Posterior Under Isotropic Marginals}

\begin{theorem}[Closed-Form Posterior Under Isotropic Marginals]
\label{thm:isotropic_posterior_simplification}
Assume isotropic marginals
\begin{align}
\Sigma_t = \sigma_t^2 \mathbf{I},
\qquad
\Sigma_{t-1} = \sigma_{t-1}^2 \mathbf{I}.
\end{align}
Then the posterior covariance simplifies to
\begin{align}
\Sigma_{t\rightarrow t-1}
=
\sigma_{t-1}^2\mathbf{I}
-
\frac{\sigma_{t-1}^4}{\sigma_t^2}
M_t^T M_t,
\end{align}
and the posterior mean simplifies to
\begin{align}
\mu_{t\rightarrow t-1}
=
\mu_{t-1}
+
\frac{\sigma_{t-1}^2}{\sigma_t^2}
M_t^T
\left(x_t - M_t \mu_{t-1}\right).
\end{align}
\end{theorem}
\begin{proof}
\begin{align}
\Sigma_{t\rightarrow t-1} &=(\Sigma_{t-1}^{-1}+M_t^T\Sigma_{t|t-1}^{-1}M_t)^{-1} \nonumber \\
&= \Sigma_{t-1} - \Sigma_{t-1}M_t^T(\Sigma_{t|t-1}+M_t\Sigma_{t-1}M_t^T)^{-1}M_t\Sigma_{t-1} \tag{c}\label{eq:woodbury_1}\\
&= \Sigma_{t-1}-\Sigma_{t-1}M_t^T\Sigma_t^{-1}M_t\Sigma_{t-1} \tag{d}\label{eq:expand_sigma}\\
&= \sigma_{t-1}^2\textbf{I}-\frac{\sigma_{t-1}^4}{\sigma_t^2}M_t^TM_t \tag{e}\label{eq:sub_iso}
\end{align}
We start from Eq.~\ref{eq:reverse_step_ddpm_matrix} derived in the previous section. In Eq.~\ref{eq:woodbury_1}, we used the Woodbury matrix identity $(A+UCV)^{-1}=A^{-1}-A^{-1}U(C^{-1}+VA^{-1}U)^{-1}VA^{-1}$ with $A=\Sigma_{t-1}^{-1}, U=M_t^T, C=\Sigma_{t|t-1}^{-1}, V=M_t$. In Eq.~\ref{eq:expand_sigma}, we substitute the value of $\Sigma_{t|t-1}$ from Eq.~\ref{eq:forward_step_ddpm_matrix}. Finally, in Eq.~\ref{eq:sub_iso} we substitute the isotropic values for $\Sigma_t$ and $\Sigma_{t-1}$.

\begin{align}
\mu_{t\rightarrow t-1} &= \Sigma_{t\rightarrow t-1}(\Sigma_{t-1}^{-1}\mu_{t-1}+M_t^T\Sigma_{t|t-1}^{-1}x_t) = \Sigma_{t\rightarrow t-1}\Sigma_{t-1}^{-1}\mu_{t-1} + \Sigma_{t\rightarrow t-1}M_t^T\Sigma_{t|t-1}^{-1}x_t \nonumber \\
&=(\Sigma_{t-1}-\Sigma_{t-1}M_t^T\Sigma_t^{-1}M_t\Sigma_{t-1})\Sigma_{t-1}^{-1}\mu_{t-1} +  \Sigma_{t\rightarrow t-1}M_t^T\Sigma_{t|t-1}^{-1}x_t\tag{f}\label{eq:sub_sigma_rev1}\\
&=(\mathbf{I}-\Sigma_{t-1}M_t^T\Sigma_t^{-1}M_t)\mu_{t-1} + \Sigma_{t\rightarrow t-1}M_t^T\Sigma_{t|t-1}^{-1}x_t \nonumber \\
&= (\mathbf{I}-\Sigma_{t-1}M_t^T\Sigma_t^{-1}M_t)\mu_{t-1} + (\Sigma_{t-1}^{-1}+M_t^T\Sigma_{t|t-1}M_t)^{-1}M_t^T\Sigma_{t|t-1}^{-1}x_t \tag{g}\label{eq:sub_sigma_rev2}\\
&=(\mathbf{I}-\Sigma_{t-1}M_t^T\Sigma_t^{-1}M_t)\mu_{t-1} + \Sigma_{t-1}^{-1}M_t^T(\Sigma_{t|t-1} + M_t\Sigma_{t-1}M_t^T)^{-1}x_t \tag{h}\label{eq:woodbury2}\\
&=(\mathbf{I}-\Sigma_{t-1}M_t^T\Sigma_t^{-1}M_t)\mu_{t-1} + \Sigma_{t-1}^{-1}M_t^T\Sigma_t^{-1}x_t \tag{i}\label{eq:expand_sigma_2}\\
&=\mu_{t-1} + \Sigma_{t-1}M_t^T\Sigma_t^{-1}(x_t-M_t\mu_{t-1})  \nonumber \\
&= \mu_{t-1} + \frac{\sigma_{t-1}^2}{\sigma_t^2}M_t^T(x_t-M_t\mu_{t-1}) \tag{j}\label{eq:sub_iso_2}
\end{align}

Starting from $\mu_{t\rightarrow t-1}$ in Eq.~\ref{eq:reverse_step_ddpm_matrix}, in Eq.~\ref{eq:sub_sigma_rev1} we substitute $\Sigma_{t\rightarrow t-1}$ from Eq.~\ref{eq:expand_sigma}, and then in Eq~\ref{eq:sub_sigma_rev2} we substitute $\Sigma_{t\rightarrow t-1}$ from Eq.~\ref{eq:reverse_step_ddpm_matrix}. In Eq.~\ref{eq:woodbury2}, we use a corollary of Woodbury identity $(A+UCV)^{-1}UC=A^{-1}U(C^{-1}+VA^{-1}U)^{-1}$, with the same substitution as described above. In Eq.~\ref{eq:expand_sigma_2}, we substitute the value of $\Sigma_{t|t-1}$ from Eq.~\ref{eq:forward_step_ddpm_matrix}, and finally, we substitute the isotropic values for $\Sigma_t$ and $\Sigma_{t-1}$.

\end{proof}

\twocolumn

{
    \small
    \bibliographystyle{ieeenat_fullname}
    \bibliography{main}
}

\end{document}